\documentclass[sigconf]{acmart}
\usepackage[utf8]{inputenc}

\usepackage{fdsymbol}
\usepackage{float}
\usepackage[all]{xy}
\usepackage{booktabs}
\usepackage{algorithmic}
\usepackage{multirow}
\usepackage{algorithm}
\usepackage{lipsum}
\usepackage{graphics}

\newcommand{\calV}{\mathcal{V}}
\newcommand{\calE}{\mathcal{E}}
\newcommand{\calG}{\mathcal{G}}
\newcommand{\calT}{\mathcal{T}}
\newcommand{\calK}{\mathcal{K}}
\newcommand{\bA}{\textbf{A}}

\newcommand{\bC}{\textbf{C}}
\newcommand{\bW}{\textbf{W}}
\newcommand{\ba}{\textbf{a}}
\newcommand{\bJ}{\textbf{J}}
\newcommand{\bx}{\textbf{x}}
\newcommand{\bX}{\textbf{X}}
\newcommand{\bI}{\textbf{I}}
\newcommand{\be}{\textbf{e}}
\newcommand{\bE}{\textbf{E}}
\newcommand{\bvarepsilon}{\boldsymbol{\varepsilon}}
\newcommand{\balpha}{\boldsymbol{\alpha}}
\newcommand{\bbeta}{\boldsymbol{\beta}}

\newtheorem{theorem}{Theorem}
\newtheorem{lemma}{Lemma}

\newtheorem{definition}{Definition}

\AtBeginDocument{%
  \providecommand\BibTeX{{%
    \normalfont B\kern-0.5em{\scshape i\kern-0.25em b}\kern-0.8em\TeX}}}

\copyrightyear{2023}
\acmYear{2023}
\setcopyright{acmlicensed}
\acmConference[KDD '23] {Proceedings of the 29th ACM SIGKDD Conference on Knowledge Discovery and Data Mining}{August 6--10, 2023}{Long Beach, CA, USA.}
\acmBooktitle{Proceedings of the 29th ACM SIGKDD Conference on Knowledge Discovery and Data Mining (KDD '23), August 6--10, 2023, Long Beach, CA, USA}
\acmPrice{15.00}
\acmISBN{979-8-4007-0103-0/23/08}
\acmDOI{10.1145/3580305.3599436}

\begin{document}
\begin{sloppypar}
\title{MM-DAG: Multi-task DAG Learning for Multi-modal Data - with Application for Traffic Congestion Analysis}

\author{Tian Lan}
\authornote{The work is done during the author's internship in SenseTime.}
\orcid{0009-0005-8331-1190}
\email{t-lan19@mails.tsinghua.edu.cn}
\affiliation{
  \institution{Tsinghua University}
  \city{Beijing}
  \country{China}
}

\author{Ziyue Li}
\orcid{0000-0003-4983-9352}
\email{zlibn@wiso.uni-koeln.de}
\affiliation{
  \institution{University of Cologne}
  \city{Cologne}
  \country{Germany}
}

\author{Zhishuai Li}
\orcid{0000-0003-3408-6300}
\email{lizhishuai@sensetime.com}
\affiliation{
  \institution{SenseTime Research}
  \city{Shanghai}
  \country{China}
}

\author{Lei Bai}
\orcid{0000-0003-3378-7201}
\email{baisanshi@gmail.com }
\affiliation{
  \institution{Shanghai AI Laboratory}
  \city{Shanghai}
  \country{China}
}

\author{Man Li}
\orcid{0000-0003-3701-7722}
\email{mlicn@connect.ust.hk}
\affiliation{
  \institution{Hong Kong University of Science and Technology}
  \country{Hong Kong}
}

\author{Fugee Tsung}
\orcid{0000-0002-0575-8254}
\email{season@ust.hk}
\affiliation{
  \institution{Hong Kong University of Science and Technology}
  \country{Hong Kong}
}

\author{Wolfgang Ketter}
\orcid{0000-0001-9008-142X}
\email{ketter@wiso.uni-koeln.de}
\affiliation{
  \institution{University of Cologne}
  \city{Cologne}
  \country{Germany}
}

\author{Rui Zhao}
\orcid{0000-0001-5874-131X}
\email{zhaorui@sensetime.com}
\affiliation{
  \institution{SenseTime Research}
  \country{China}
}

\author{Chen Zhang}
\orcid{0000-0002-4767-9597}
\authornote{Corresponding author.}
\email{zhangchen01@tsinghua.edu.cn}
\affiliation{
  \institution{Tsinghua University}
  \city{Beijing}
  \country{China}
}

\renewcommand{\shortauthors}{Lan, et al.}

\begin{abstract}
  This paper proposes to learn Multi-task, Multi-modal Direct Acyclic Graphs (MM-DAGs), which are commonly observed in complex systems, e.g., traffic, manufacturing, and weather systems, whose variables are multi-modal with scalars, vectors, and functions. This paper takes the traffic congestion analysis as a concrete case, where a traffic intersection is usually regarded as a DAG. In a road network of multiple intersections, different intersections can only have some \textit{overlapping and distinct} variables observed. For example, a signalized intersection has traffic light-related variables, whereas unsignalized ones do not. This encourages the multi-task design: with each DAG as a task, the MM-DAG tries to learn the multiple DAGs jointly so that their consensus and consistency are maximized. To this end, we innovatively propose a multi-modal regression for linear causal relationship description of different variables. Then we develop a novel Causality Difference ($CD$) measure and its differentiable approximator. Compared with existing SOTA measures, $CD$ can penalize the causal structural difference among DAGs with distinct nodes and can better consider the uncertainty of causal orders. We rigidly prove our design's topological interpretation and consistency properties. We conduct thorough simulations and one case study to show the effectiveness of our MM-DAG. The code is available under \url{https://github.com/Lantian72/MM-DAG}. 
\end{abstract}

\ccsdesc[500]{Mathematics of computing~Causal networks}
\ccsdesc[500]{Computing methodologies~Multi-task learning}



\keywords{causal structure learning, multi-task learning, multi-modal data, directed acyclic graph}

\maketitle
\section{Introduction}

 Directed Acyclic Graph (DAG) is a powerful tool for describing the underlying causal relationships in a system. One of the most popular DAG formualtions is Bayesian Network (BN)\cite{zheng2018dags}.  It has been widely applied to the biological, physical, and social systems \cite{long2011urban, velikova2014exploiting, ruz2020sentiment}. In a DAG, 
nodes represent variables, and directed edges represent causal dependencies between nodes. By learning the edges and parameters of the DAG, the joint distribution of all the variables can be analyzed.

\begin{figure}[t]
    \centering
    \includegraphics[width=\columnwidth]{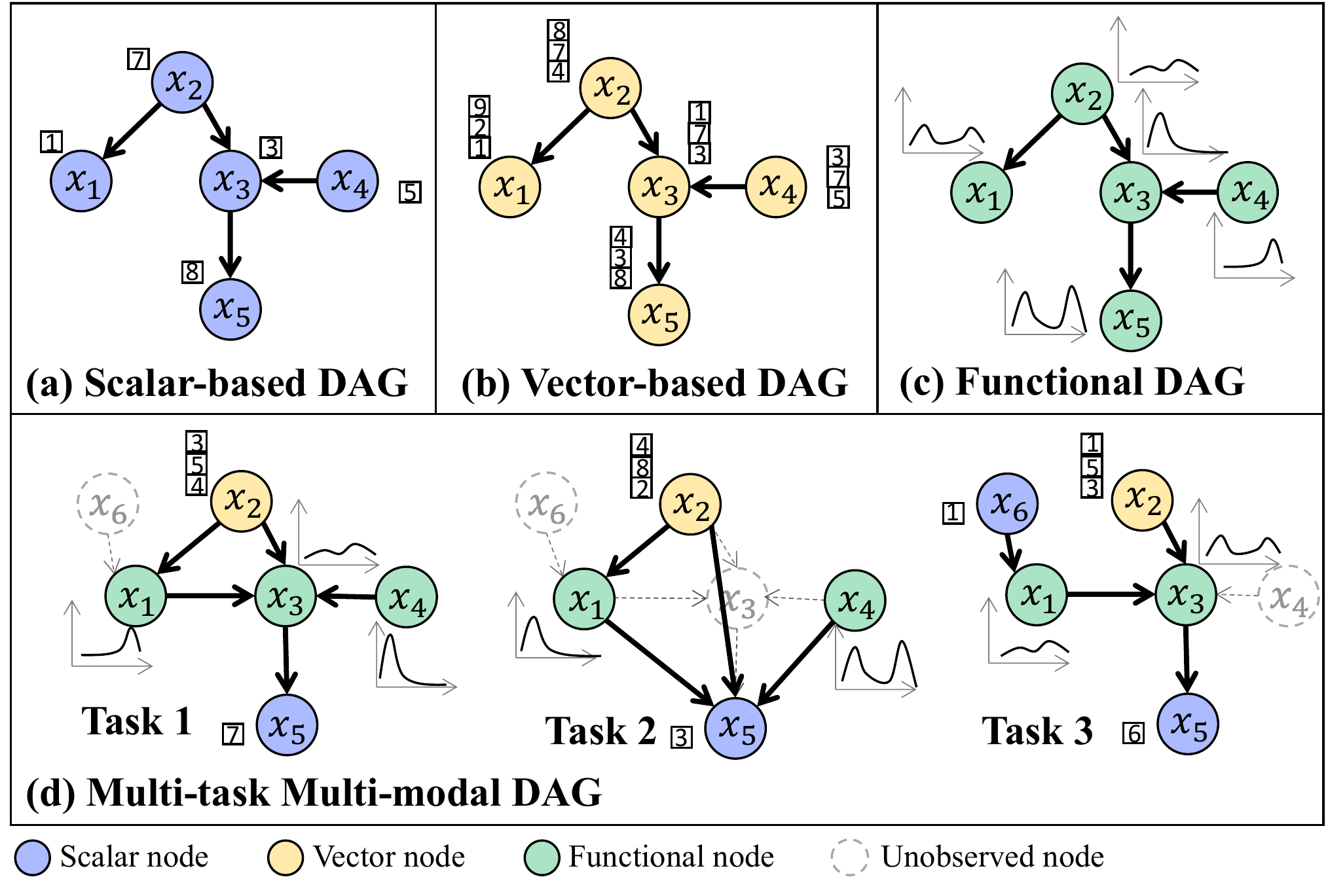}
    \caption{DAGs of (a) scalar, (b) vector, (c) functions v.s. (d) MM-DAG. Each node denotes a variable, and the directed edge means the causal dependence. The classical DAG assumes homogeneous (uni-modal) node variables, especially as scalars. MM-DAG conforms to reality better where node variables are versatile (multi-modal), and each task has overlapping and distinct nodes.}
    \label{fig:MM-DAG1}
    \vspace{-10pt}
\end{figure}

Urban traffic congestion becomes a common problem in metropolises, as urban road network becomes complicated and vehicles increase rapidly. Many factors will cause traffic congestion, such as Origin-Destination (OD) demand, the cycle time of traffic lights, weather conditions, or a road accident. Causal analysis of congestion has been highly demanded in applications of intelligent transportation systems. There is emerging research applying classical DAGs for modeling the probabilistic dependency structure of congestion causes and analyzing the probability of traffic congestion given various traffic condition scenarios \cite{kim2016diagnosis,afrin2021probabilistic, luan2022traffic}. 
When mining the causality for traffic congestion, as the classical DAG-based solution, a traffic intersection is usually regarded as a DAG, whereas different congestion-related traffic variables (e.g., lane speed and signal cycle length) are treated as nodes. However, there are still some challenges to be solved. 

(1) \textbf{Multi-mode}: First, so far, to our best knowledge, all the current DAGs consider each node as a scalar variable, which may deviate from reality: 
In complex systems such as transportation, variables are common in different modes, i.e., \textit{scalar}, \textit{vector}, and \textit{function}, due to the variables' innate nature and/or being collected from different kinds of sensors, as shown in Fig. \ref{fig:MM-DAG1}.(a)-(c). 
A \textit{scalar} node is defined as a node that only has a one-dimensional value for each sample, e.g., the cycle time of traffic lights is usually fixed and scarcely tuned. So its signals are sampled at low frequency, and only one data point is fed back in one day. A \textit{vector} node instead records a vector with higher but \textbf{finite} dimensions, e.g., the congestion indicator variable is calculated per hour, thus a fixed dimension of 24 per day. A \textit{functional} node records a random function for each sample, with the function being high dimensional and also \textbf{infinite}, e.g., the real-time mean speed of lanes can be recorded every second and its dimension goes to infinity for one day. 
\textit{So far, there is no DAG modeling able to deal with multi-modal data. } 

(2) \textbf{Multi-task with Overlapping and Distinct Variables}: 
We define a \textit{task} as a DAG learning, e.g., for each intersection in the traffic case. In complex systems, different tasks can only have some overlapping variables, with some particular variables only distinctly occurring in some specific tasks. We define this officially as \textit{overlapping and distinct} observations (variables). As such, each task can be regarded as only observing a unique subset of all possible variables. This may be due to their different experiences and hardware availability. 
For example: 1) Distinct: a signalized intersection (e.g., Task 3) has the node variable related to traffic light parameter (e.g., $x_6$), such as phase length, whereas a road segment (e.g., Task 1) and an unsignalized intersection (e.g., Task 2) do not have $x_6$; 2) Overlapping: Task 1-3 all have $x_1, x_2, x_5$ in Fig. \ref{fig:MM-DAG1}.(d). 
The different availability of nodes in each task is the \textbf{dissimilarity} of our multi-task setting. In multi-task learning \cite{zhang2021survey}, two important concepts are exactly the dissimilarity and similarity of tasks. 

\textbf{(3) Consistent Causal Relations:} Despite the different nodes on each task, we assume the causal relations of each DAG should be almost \textit{consistent} and \textit{non-contradictory}. For instance, $x_1$ is the cause of $x_3$ in Task 1, and this causal relation is not likely to be reversed in another task. This is because although with different subsets of the nodes, the DAGs are sharing and reflecting the similar fundamental and global causal reasoning of the system. 
This fundamental causal reasoning is usually consistent, usually due to the inherent physical, topological, biochemical properties, and so on. The consistent causal reasoning commonly shared by all the tasks is the \textbf{similarity} of our multi-task setting. However, it is worth mentioning that because nodes vary in each task, the corresponding causal relation structure will undoubtedly adapt, even with significant differences sometimes. For example, as illustrated in Task 1 and 2 in Fig \ref{fig:MM-DAG1}.(d), because node $x_3$ is uninvolved in Task 2, all the edges from its predecessors $\{x_1, x_2, x_4 \}$ will be transited to its successors ($x_5$) directly, rendering a big difference of edges (yet still consistent).

The core challenge is thus to define structure differences between DAGs with different but overlapping sets of nodes, yet still learn the causal reasoning consistently. To this end, it is essential to learn these tasks jointly so that each DAG provides complementary information mutually and learns toward globally consistent causal relations. On the contrary, if separately learned, the causal structure of each task could be partial, noisy, and even contradicting.

Motivated by the three challenges, this paper aims at constructing DAG for multi-modal data and developing a structure inference algorithm in a multi-task learning manner, where the node sets of different tasks are overlapping and distinct. To achieve it, three concrete questions need to be answered: (1) how to extract information for nodes with different dimensions and model their causal dependence? (2) how to measure the differences in causal structures of DAGs across tasks? (3) how to design a structural learning algorithm for DAGs of different tasks? 

Unfolded by solving the above questions, we are the first to construct multi-task learning for multi-modal DAG, named MM-DAG. First, we construct a linear multimode-to-multimode regression for causal dependence modeling of multi-modal nodes. Then we develop a novel measure to evaluate the causal structure difference of different DAGs. Finally, a score-based method is constructed to learn the DAGs across tasks with overlapping and distinct nodes such that they have similar structures. Our contributions are:


\begin{itemize}
    \item We propose a multimode-to-multimode regression 
    to represent the linear causal relationship between variables. It can deal with nodes of scalar, vector, and functional data.
    \item We develop a novel measure, i.e., \textbf{C}ausality \textbf{D}ifference (\textit{CD}), to evaluate the structure difference between pairwise DAGs with overlapping and distinct nodes. It can better handle graphs with distinct nodes and consider the uncertainty of causal order. A differentiable approximator is also proposed for its better compatiblity with our learning framework.  
    \item 
    We conduct a score-based structure learning framework for MM-DAGs, with our novelly-designed differentiable $CD$ function to penalize DAGs' structure difference. Most importantly, we also prove theoretically the topological interpretation and the consistency of our design.
    \item We apply MM-DAG in traffic condition data of different contexts to infer traffic congestion causes. The results provide valuable insights into traffic decision-making. 
\end{itemize}

It is to be noted that considering even for the most commonly used causal structural equation models (SEM), there is no work of multi-task learning for DAG with multimode data. Hence we focus on linear multimode-to-multimode regression, as the first extension of SEM to multi-modal data. We hope to shed light upon this research field since the linear assumption is easy to comprehend. However, our proposed CD measure and multi-task framework can be easily extended to more general causal models, including some nonlinear or deep learning models, with details in Sec. \ref{sec: non-linear}.

The remainder of the paper is organized as follows. Section \ref{sec:related-work} reviews the current work about DAG, multi-task learning, and traffic congestion cause analysis. Section \ref{sec:methodology} introduces the model construction of MM-DAG in detail and discusses how to extend our model to nonlinear cases. Section \ref{sec:experiment} shows the experimental results, including the synthetic data and traffic data by SUMO simulation. Conclusions and future work are drawn in Section \ref{sec:conclusion}. 

\section{Related Work} \label{sec:related-work}


\subsection{DAG Structure Learning Algorithm}
Structure learning for DAG, i.e., estimating its edge sets and adjacency matrix, is an important and well-studied research topic. The current methods can be categorized into constraint-based algorithms and score-based algorithms. (1) Constraint-based algorithms employ statistical hypothesis tests to identify directed conditional independence relationships from the data and construct a BN structure that best fits those directed conditional independence
relationships, including PC \citep{spirtes2000causation}, rankPC \citep{harris2013pc}, and fast causal inference \citep{spirtes2000causation}. However, constraint-based algorithms are built upon the assumption that independence tests should accurately reflect the (in)dependence modeling mechanism, which is generally difficult to be satisfied in reality. As a result, these methods suffer from error propagation, where a minor error in the early phase can result in a very different DAG. 
(2) For score-based methods, a scoring function, such as fitting mean square error or likelihood function, is constructed for
evaluating the goodness of a network structure. Then a search procedure for the highest-scored structure, such as stochastic local research \citep{chickering2002optimal,nandy2018high} or dynamic programming \citep{koivisto2004exact}, is formulated as a combinatorial optimization problem. 
However, these methods are still very unpractical and restricted for large-scale problems.

Some other algorithms for structure learning have been developed to reduce computation costs recently. The most popular one is NoTears \citep{zheng2018dags}. It represents acyclic constraints by an algebraic characterization, which is differentiable and can be added to the score function. The gradient-based optimization algorithm can be used for structure learning. Most recent DAG structural learning studies follow the insights of NoTears \cite{bhattacharya2021differentiable, ng2020role}.  Along this direction, there are also emerging works applying the Notears constraint into nonlinear models for nonlinear causality modeling. The core is to add the Notears constraint into original nonlinear model's loss function to guarantee the graph's acyclic property. For example, \citet{zheng2020learning} proposes a general nonparametric modeling framework to represent nonlinear causal structural equation model (SEM). \citet{yu2019dag} proposes a deep graph convolution model where the graph represents the causal structure.
\subsection{Multi-task Learning Algorithm for DAG}
Multi-task learning is common in complex systems such as manufacturing and transport \cite{zhang2021survey, li2022profile, shen2023multi}. In DAG, multiple-task modeling is first proposed for tasks with the same node variables and similar causal relationships \citep{niculescu2007inductive}. To learn different tasks jointly, it penalizes the number of different edges among tasks and uses a heuristic search to find the best structure. \citet{oyen2012leveraging} further introduces a task-relatedness metric, allowing explicit control of information sharing between tasks into the learning objective. 
\citet{oyen2013bayesian} proposes to penalize the number of edge additions which breaks down into local calculations, i.e., the number of differences of parent nodes for different tasks, to explore shared and unique structural features among tasks in a more robust way. \citet{oates2016exact} proposes to model multiple DAGs by encoding the relationships between different DAGs into an undirected network.

As an alternative solution for multi-task graph learning, hidden structures are exploited to fuse the information across tasks. The idea is first to find shared hidden structures among related tasks and then treat them as the structure penalties in the learning step\citep{oates2016exact}. Later, to better address the situation that the shared hidden structure comes from different parts of different DAGs, \citet{zhou2017multiple} proposes to use a non-negative matrix factorization method to decompose the multiple DAGs into different parts and use the corresponding part of the shared hidden structure as a penalty in different learning tasks. However, these methods penalize graph differences based on their general topology structure, which yet does not represent causal structure. To better add a penalty from a causality perspective, \citet{chen2021multi} proposes to regularize causal orders of different tasks to be the same. However, all the above methods should assume different tasks share the same node set and cannot be applied for tasks with both shared and specific nodes.



\subsection{Congestion Causes Analysis}
Smart transport has been an essential chapter, yet with many works focusing on demand prediction \cite{li2020tensor,li2020long}, trajectory \cite{li2022individualized, ziyue2021tensor, mao2022jointly}, or etc. Congestion root analysis instead should gain more attention since it is safety-related. It uses traffic variables to classify congestion into several causes. \citet{chow2014empirical} uses linear regression to diagnose and assign observed congestion to various causes. \citet{al2017distributed,al2019cooperative} propose a real-time classification framework for congestion by vehicular ad-hoc networks. 
\citet{afrin2021probabilistic} uses BN to estimate the conditional probability between variables. \citet{kim2016diagnosis} divides the nodes in BN into three groups, representing the environment, external events, and traffic conditions, and uses the discrete BN to estimate the causal relationships between nodes. 
However, the studies above did not involve the correlations between different congestion causes and just classified the congestion into several simple categories. Besides, BN \cite{xing2015discovering, fan2019prediction} has also been applied to congestion propagation \cite{daganzo1994cell, long2011urban, zhang2012ctm}. Other propagation models include the Gaussian mixture model \citep{sun2006bayesian}, congestion tree structure \citep{zhou2017multiple}, and Bayesian GCN \citep{luan2022traffic}. Yet we focus on the root causes analysis instead of congestion propagation.


\section{Methodology}
\label{sec:methodology}
We assume there are in total $L$ tasks. For each task $l=1,\ldots, L$, we have $P_{l}$ nodes, with the node set $\calV_{l}=\{1,\ldots,P_{l}\}$. The node $j$ in task $l$ is denoted as $x_{j (l)} \in \mathbb{R}^{T_{j(l)}}$ representing a variable with dimension $T_{j(l)}$. Depending on $T_{j(l)}$, a node $x_{j(l)}$ can represent multi-modal data: $x_{j(l)}$ is a scalar when $T_{j(l)}=1$, a vector when $T_{j(l)} \in [2,\infty)$, and a function if $T_{j(l)} = \infty$. We aim to construct a DAG for task $l$, i.e., $\calG_{l}=(\calV_{l},\calE_{l})$, where the edge set $\calE_{l}\in \mathbb{R}^{P_{l} \times P_{l}}$ and an edge $(j, k)$ represents a causal dependence $x_{j(l)} \to x_{k(l)}$. 

In Section \ref{sec:Multimode-dag}, we temporarily focus on a single task and assume that the causal structure is known. We construct a probabilistic representation of multi-mode DAG by \textbf{mul}ti\textbf{mo}de to \textbf{mul}ti\textbf{mo}de regression, called \textit{mulmo2} for short. Then in Section \ref{sec:structural-learning}, we consider all the $L$ tasks and propose a score-based objective function for structural learning. Its core is how to measure and penalize the causal structure difference of different tasks. Here we provide a novel measure, \textit{CD}, together with its differentiable variant \textit{DCD}, which tries to keep the transitive causalities among overlapping nodes of different tasks to be consistent, as elaborated in Section \ref{sec:design-penalty}. Finally, in Section \ref{sec:algorithm}, we give the optimization algorithm for solving the score-based multi-task learning. 

\subsection{Multi-mode DAG with Known Structure} 
\label{sec:Multimode-dag}
We temporarily focus on single-task learning. For notation convenience, we remove the subscript $l$ of $\calG_{l}$ in Section \ref{sec:Multimode-dag}. Besides, we temporarily assume that causal structure $\calE$, i.e., the parents of each node, are known. We denote the parents of the node $j$ as $pa_j=\{j'| (j',j)\in \calE)\}$. Thus, the joint distribution for sample $n$ is the production of the conditional distribution of each node.
\begin{equation}
\begin{split}
    &p(x^{(n)}_1,\ldots,x^{(n)}_P) = \prod_{j=1}^P p(x^{(n)}_j|pa_{j})
\end{split}
\end{equation}

When the multi-mode nodes have finite dimensions, the relationships among a multi-mode node $j$ and its parent $j'\in pa_j$ can be represented by the following \textit{mulmo2} regression model:
\begin{equation} \label{equ:multimodel-regression}
\begin{split}
    x^{(n)}_{j} &= \sum_{j' \in pa_j} \ell_{j'j}(x^{(n)}_{j'}) + e_{j}^{(n)}
\end{split}
\end{equation}
where $\ell_{j'j}$ is the linear transform of $x_{j'}$ for $(j',j) \in \calE$, $e_j^{(n)}$ is the noise of $x_j^{(n)}$ with the expectation $\mathbb{E}[e_j^{(n)}]=0$. 

We consider $\ell_{j'j}$ for four cases by whether $T_j$ or $T_{j'}$ is infinite, as shown in Fig. \ref{fig:MM-DAG}. If $T_j$ or $T_{j'}$ is infinite, we consider $x_j$ or $x_{j'}$ as a functional variable. From the following, by abuse of notation, we define a vector node as $\bx_{j}$, a function node as $x_{j}(t), t \in \Gamma$. Without loss of generality, we assume the $\Gamma=[0,1]$ is a compact time interval for all the function nodes.

\begin{figure}[t]
    \centering
    \includegraphics[width=0.52\columnwidth]{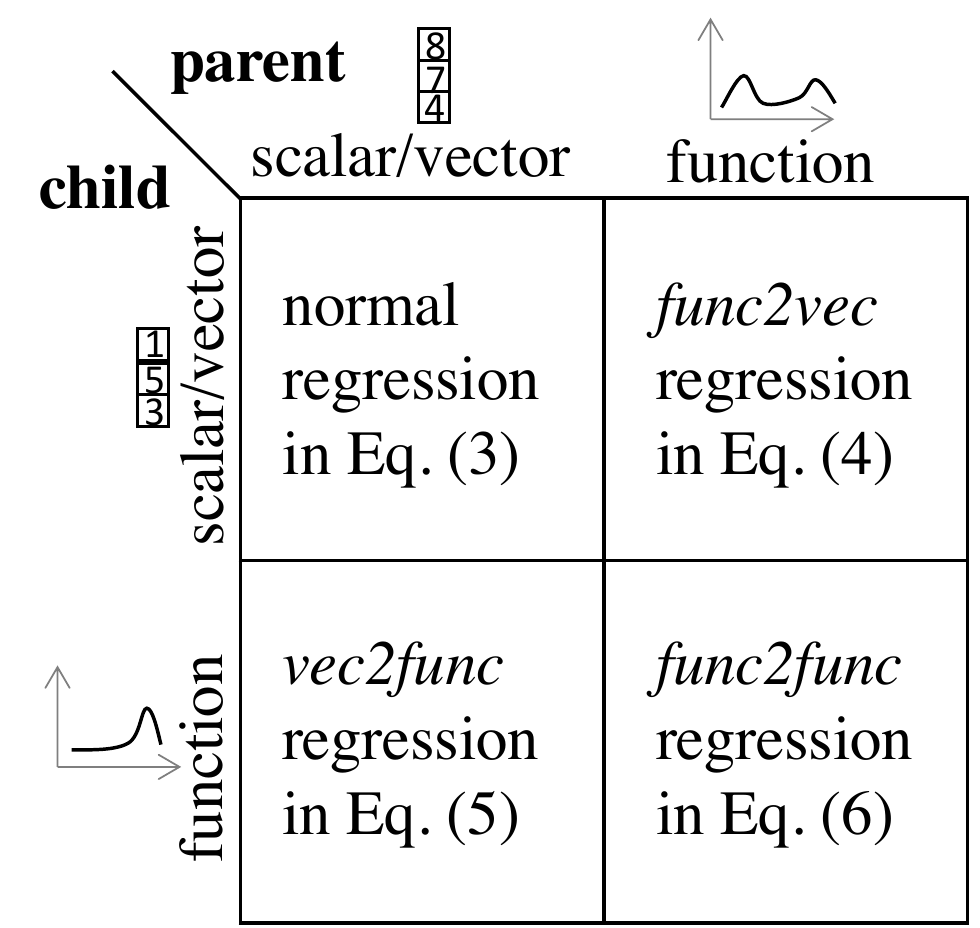}
    \caption{Four components of \textit{mulmo2} regression.}
    \label{fig:MM-DAG}
    \vspace{-10pt}
\end{figure}

\textbf{Case 1:} Both of two nodes have finite dimensions, i.e., $T_j, T_{j'} < \infty$. Then the transition equation is a normal regression:
\begin{equation} \label{equ:finite-finite}
    \left( \ell_{j'j}(\bx_{j'}^{(n)})\right)_t = \sum_{s=1}^{T_{j'}} c_{j'jst} \bx_{j's}^{(n)}, t= 1,\ldots, T_j.
\end{equation}
Here $c_{j'jst}$ is the coefficient of component $s$ of the vector $\bx_{j'}$ to component $t$ of the vector $\bx_j$ and $(j',j) \in \calE$.

\textbf{Case 2:} $\bx_j$ has finite dimensions (vector) and $x_{j'}(t)$ has infinite dimensions (function), i.e., $T_j < \infty, T_{j'} = \infty$. Then $\ell_{j'j}$ is:
\begin{equation} \label{equ:finite-infinite}
\begin{split}
    \left(\ell_{j'j}(x_{j'}^{(n)}(s))\right)_t = \int_0^1 \gamma_{j'jt}(s) x^{(n)}_{j'}(s) {\rm d}s,t= 1,\ldots, T_j.
\end{split}
\end{equation}
Here $\gamma_{j'jt}(s)$ is the coefficient function for component $t$ in vector $
\bx_{j}$ and $(j',j) \in \calE$.

\textbf{Case 3:} $x_j(t)$ has infinite dimensions (function) and $\bx_{j'}$ has finite dimensions (vector), i.e., $T_j = \infty, T_{j'} < \infty$. In this case, the linear regression between vector-to-function regression is:
\begin{equation} \label{equ:infinite-finite}
\begin{split}
    \ell_{j'j}(\bx_{j'}^{(n)})(t) &= \sum_{s=1}^{T_{j'}} \gamma_{j'js}(t) \bx^{(n)}_{j's},
\end{split}
\end{equation}
where $\gamma_{j'js}(t)$ is the coefficient function for $s$-th component in vector $\bx_{j'}$ and $(j',j) \in \calE$.

\textbf{Case 4:} Both of two nodes have infinite dimensions, i.e., $T_j, T_{j'} = \infty$, Then, the linear function-to-function (\textit{func2func}) regression is:

\begin{equation} \label{equ:infinite-infinite}
\begin{split}
    \ell_{j'j}(x_{j'}^{(n)})(t)&= \int_0^1 \gamma_{j'j}(t,s) x^{(n)}_{j'}(s) {\rm d}s,
\end{split}
\end{equation}
where $\gamma_{j'j}(t,s)$ is the coefficient function for $(j',j) \in \calE$.

For any node $j\in\{j\in \calV|T_j=\infty\}$, $x_j(t)$ is in infinite dimensions and hard to be estimated directly. It is common to decompose them into a well-defined continuous space for feature extraction:
\begin{equation} \label{equ:FPCA-x}
    x_j^{(n)}(t) = \sum_{k=1}^{K_j} \alpha_{jk}^{(n)} \beta_{jk}(t) + \varepsilon_j^{(n)}(t),
\end{equation}
where ${\beta_{jk}(t)}$ is the orthonormal functional basis, with $\int_0^1 \beta_{jk}(t)^2{\rm d}t = 1$ and $\int_0^1\beta_{jk}(t)\beta_{jk'}(t){\rm d}t = 0$ for $k,k'=1,\ldots K_j$ and $k\ne k'$, $\alpha_{jk}^{(n)}$ is the corresponding coefficient. $\alpha_{jk}^{(n)}$ and $\beta_{jk}(t)$ can be obtained by Functional Principal Component Analysis (FPCA) \cite{yao2005functional}, and $\varepsilon_j^{(n)}(t)$ is the residual of FPCA.

After decomposing the functional variables $x_j(t)$, we describe transition $\gamma$ in Cases 2, 3, and 4 using the corresponding basis set:
\begin{equation} \label{equ:FPCA-gamma}
\small
\begin{split}
    \gamma_{j'jt}(s)&=\sum_{k'=1}^{K_{j'}} c_{j'jtk'}\beta_{j'k'}(s) \\
    \gamma_{j'js}(t)&=\sum_{k=1}^{K_j} c_{j'jsk}\beta_{jk}(t) \\
    \gamma_{j'j}(t,s)&=\sum_{k=1}^{K_j} \sum_{k'=1}^{K_{j'}} c_{j'jk'k}\beta_{jk}(t)\beta_{j'k'}(s). \\
\end{split}    
\end{equation}

Plugging Eqs. (\ref{equ:FPCA-x}) and (\ref{equ:FPCA-gamma}) into Eqs. (\ref{equ:multimodel-regression}), (\ref{equ:finite-finite}), (\ref{equ:finite-infinite}), (\ref{equ:infinite-finite}) and (\ref{equ:infinite-infinite}), we have the general expression of our \textit{mulmo2} regression:
\begin{equation}\label{equ:transition}
    \ba^{(n)}_j = \sum_{j' \in pa_j} \bC_{j'j}^T \ba^{(n)}_{j'} + \tilde{\be}^{(n)}_j,
\end{equation}
where 
\begin{equation} \label{equ:A}
    \textbf{a}^{(n)}_j = 
    \begin{cases}
       \bx^{(n)}_j \in \mathbb{R}^{T_j}  & T_j < \infty \\
       \balpha^{(n)}_j \in \mathbb{R}^{K_j} & T_j = \infty 
    \end{cases} ,
\end{equation}
$\balpha_{j}^{(n)}=[\alpha_{j1}^{(n)},\ldots,\alpha_{jK_j}^{(n)}]$ is the PC score of node $j$ in sample $n$, and $\bC_{j'j}\in \mathbb{R}^{d(\ba_{j'}) \times d(\ba_j)}$ represents the transition matrix from node $j'$ to node $j$, with $(\bC_{j'j})_{uv} = c_{j'juv}$ and $d(\ba_j)$ as the dimensions of $\ba_j$. $\tilde{\be}_j^{(n)} \in \mathbb{R}^{d(\ba_j)}$ is the noise, with 1) $\tilde{\be}_j^{(n)} = \be_j^{(n)}$ if $T_j<\infty$, and 2) $(\tilde{\be}_{j}^{(n)})_{k} = \int_0^1 e_j^{(n)}(t) \beta_{jk}(t) {\rm d}t$, if $T_j=\infty$. We have $\mathbb{E}[\tilde{\be}_j^{(n)}] = 0$ in these two cases.
It is to be noted that we can also conduct PCA to perform dimension reduction for vector variables like $\bx_j^{(n)} = \sum_{k=1}^{K_j} \alpha_{jk}^{(n)} \bbeta_{jk} + \bvarepsilon_j^{(n)}$, and replace the finite cases in Eq. (\ref{equ:A}) by $\ba_j^{(n)}=\balpha_j^{(n)} \in \mathbb{R}^{K_j}$, $K_j \leq T_j$. 

We assume noise $\tilde{\be}_j$ follows Gaussian distribution independently and interpret Eq. (\ref{equ:transition}) as linear Structural Equation Model (SEM):
\begin{equation} \label{equ:SEM}
    \ba^{(n)} = \bC^T \ba^{(n)} + \tilde{\be}^{(n)}.
\end{equation}
Here $\ba^{(n)} = [\ba^{(n)}_1,\ldots,\ba^{(n)}_{P}]\in \mathbb{R}^{M}$, $\tilde{\be}^{(n)}=[\tilde{\be}_{1}^{(n)},\ldots,\tilde{\be}_{j}^{(n)}]\in \mathbb{R}^{M}$ is the noise vector, and
 $\bC = [\bC_{j'j}] \in \mathbb{R}^{M\times M}$ is the combined matrix where $M = \sum_j d(\ba_j)$.

\subsection{Multi-task Learning of Multi-mode DAG}
\label{sec:structural-learning}
Now we discuss how to estimate the DAG structures for all the tasks.
First, we introduce the concept of causal order $\pi(\cdot)$, which informs possible ``parents'' of each node. It can be represented by a permutation over $1, 2, \ldots, P$. If we sort the nodes set by
their causal orders, the sorted sequence satisfies that the left node is a parent or independent of the right node. 
A graph $\calG=(\calV,\calE)$ is consistent with a causal order $\pi$ if and only if:
\begin{equation} 
    (i,j) \in \calE \Rightarrow \pi(i) < \pi(j).
\end{equation}

In SEM of Eq. (\ref{equ:SEM}), we focus on estimating the transition matrix $\bC$ and its causal order $\pi$. The non-zero entries of the matrix $\bC$ denote the edges of the graph $\calG=(\calV,\calE)$ that must consistent with $\pi$, i.e.,  $\|\bC_{ij}\|_F^2>0 \Rightarrow \pi(i) < \pi(j)$. We denote $\bW_{ij} = \|\bC_{ij}\|_F^2$ to represent the weight of edge from node $i$ to node $j$, where $\bW_{ij}>0$ means $(i,j)\in \calE$. 
Based on the acyclic constraint proposed by NoTears \cite{zheng2018dags}, our score-based estimator of single-task is:
\begin{align}
    \hat{\bC} = \underset{\bC}{\arg \min} \frac{1}{2N} \|\bA - \bA\bC\|_F^2 + \lambda \|\bC\|_1 \label{equ:score-based}\\
    {\rm subject\ to\ } h(\bW) = {\rm tr}(e^{\bW}) - P = 0 .\label{equ:acylic}
\end{align}
where $\bA=[\ba^{(1)},\ldots,\ba^{(N)}]^{T}\in \mathbb{R}^{N\times M}$.
For all the tasks $l=1,\ldots,L$, we denote their corresponding SEMs as:
\begin{equation} 
    \bA_{(l)} = \bC_{(l)}^T \bA_{(l)} + \bE_{(l)},
\end{equation}
where $\bA_{(l)} \in \mathbb{R}^{N_l \times M_l}$, $\bC_{(l)} \in \mathbb{R}^{M_l \times M_l}$, and $\bE_{(l)}=[\tilde{\be}_{(l)}^{(1)},\ldots,\tilde{\be}_{(l)}^{(N_l)}]^{T} \in \mathbb{R}^{N_l \times M_l}$ is the noise matrix of $N_l$ samples of task $l$. 

The core of multi-task learning lies in how to achieve information sharing between tasks. To this end, we add the penalty term to penalize the difference between pairwise tasks and derive a score-based function of multi-task learning as follows: 
\begin{equation} \label{equ:multitask-scorebased}
    \begin{split}
            \hat{\bC}_{(1)},...,\hat{\bC}_{(L)} & = \underset{\bC_{(1)},...,\bC_{(L)}}{\arg \min} \sum_{l=1}^L \frac{1}{2N_l} \|\bA_{(l)} - \bA_{(l)}\bC_{(l)}\|_F^2\\
    & + \rho \sum_{l_1,l_2} s_{l_1,l_2} DCD(\bW_{(l_1)}, \bW_{(l_2)})  + \lambda \sum_{l=1}^L \|\bC_{(l)}\|_1 \\
    &{\rm s.t. \ } h(\bW_{(l)}) = {\rm tr}(e^{\bW_{(l)}}) - P_l = 0, \forall l
    \end{split}
\end{equation}
where $\bW_{(l)ij} = \|\bC_{(l)ij}\|_F^2$, $s_{l_1,l_2}$ is the given constant reflecting the similarity between tasks $l_1$ and $l_2$. The penalty term $DCD(\bW_{(l_1)}, \bW_{(l_2)})$ is defined as  \textbf{D}ifferentiable  \textbf{C}ausal \textbf{D}ifference of the DAGs between task $l_1$ and task $l_2$ (discussed in Section \ref{sec:design-penalty}). $\rho$ controls the penalty of the difference in causal orders, where larger $\rho$ means less tolerance of difference.  $\lambda$ controls the {${L}_1$-norm} penalty of $\bC_{(l)}$ which guarantees that $\bC_{(l)}$ is sparse.

\subsection{Design the Causal Difference} 
\label{sec:design-penalty}
We propose a novel differentiable measure to quantify causal structure difference between two DAGs. First, we introduce the current most commonly used measures for graph structure difference. They are limited when formulating the \textit{transitive} causality between two DAGs (details below). Then we introduce the motivation of \textbf{C}ausal \textbf{D}ifference measure $CD$ and its definition. Finally, we propose $DCD$ as the differentiable $CD$ and discuss its asymptotic properties. 

\textbf{Current metrics for graph structure difference} include spectral distances, matrix distance, feature-based distance \cite{wills2020metrics}. A simple idea is to directly count how many edges are different between two graphs, denoted as $\Delta(\calG_u, \calG_v)$. It is a special case of matrix distance $\|\bW_u - \bW_v\|_0$, and $\bW_u, \bW_v$ is the adjacency matrix of graph $\calG_u, \calG_v$. $\Delta(\calG_u, \calG_v)$ defines the edge difference of $\calG_u$ and $\calG_v$:
\begin{equation} \label{equ:difference}
    \Delta(\calG_u,\calG_v) = \sum_{i \in \calV_u \cap \calV_v} \sum_{j \in \calV_u \cap \calV_v} \mathbb{I}(\mathbb{I}((i,j)\in \calE_u) \ne \mathbb{I}((i,j) \in \calE_v) )
\end{equation}
where $\calV_u, \calV_v \subseteq \{1, 2, \ldots, P \}$ are the node sets of the graph $\calG_u, \calG_v$, respectively. $\mathbb{I}(\cdot)$ is the indicator function. 
If node $i$ and node $j$ both appear in $\calV_u, \calV_v$, the difference is increased by one if the edge $(i,j)$ appears in $\calE_u$ but not in $\calE_v$ and vice versa. However, $ \Delta(\calG_u, \calG_v)$ does not consider the edges of distinct nodes of $\calV_u, \calV_v$. This is reasonable since, in our context of multi-task learning, we only need to penalize the model difference for the shared parts, i.e., the graph structure for the overlapping nodes. 

\textbf{A novel measure considering transitive causality:} $ \Delta(\calG_u, \calG_v)$ performs well if we only focus on the graph structure difference. However, it cannot reveal the transitivity of causal relationships in graphs. We use the three graphs $\calG_a, \calG_b, \calG_c$ in Fig. \ref{fig:situation} to demonstrate this point.
\begin{figure}[t]
    \centering
    \includegraphics[width=0.9\columnwidth]{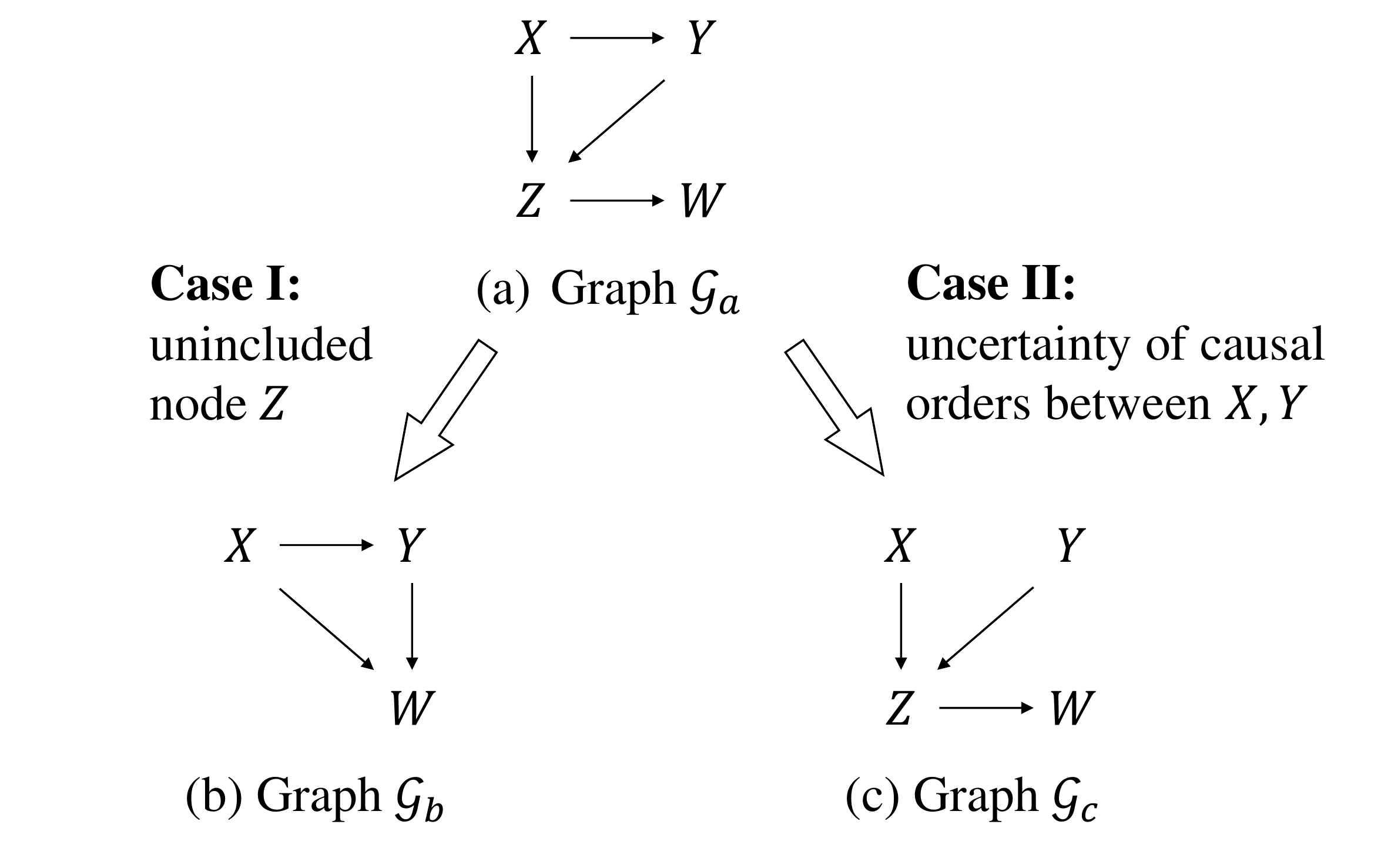}
    \caption{Unincluded nodes \& Uncertainty of causal orders.}
    \label{fig:situation}
\end{figure}

(1) \textbf{Case I:} The difference between $\calG_a$ and $\calG_b$. In this case $\Delta(\calG_a, \calG_b) = 2$ since the edges $X \rightarrow W, Y \rightarrow W$ appear in $\calE_a$, not in $\calE_b$. But from another perspective, if we sort the nodes set by their causal orders, the sorted sequence in $\calG_a$ is $X,Y,Z,W$, and the sorted sequence in $\calG_b$ is $X,Y,W$. If we remove $Z$ in $\calG_a$, the sorted sequence of $\calG_a$ and $\calG_b$ are exactly the same. The edge difference between $\calG_a$ and $\calG_b$ is due to the transitive causality passing $Z$, which is excluded in $\calV_b$. Thus, the ideal \textbf{C}ausal \textbf{D}ifference measure should be $CD(\calG_a, \calG_b)=0$, which is formally defined in Def. \ref{def:causal-difference}. 

(2) \textbf{Case II:} \label{case:2}The difference between $\calG_a$ and $\calG_c$. To solve the problem of \textbf{Case I}, at first glance, we can use causal order \cite{chen2021multi} and kernels for permutation \cite{jiao2015kendall} as a causal difference measure directly. However, it has an uncertainty problem, as shown in Fig. \ref{fig:situation}.(b). In $\calG_c$, the sorted sequence is either $X,Y,Z,W$ or $Y,X,Z,W$, which are equivalent. But in $\calG_a$, the sorted sequence is unique $X,Y,Z,W$. This difference is caused by that there is an edge $X \rightarrow Y$ in $\calG_a$, which determines the causal order that $\pi(X)<\pi(Y)$, but not in $\calG_c$. In this case, the causal difference measure between the two graphs should be considered, i.e., $CD(\calG_a,\calG_c) > 0$.

\textbf{Our design:} The two cases mentioned above motivate us to propose a new measure to evaluate the causal difference. Instead of using causal order, which is a one-dimensional sequence, here we define a transitive causal matrix to better consider causal order with uncertainty.
\begin{definition}[Transitive causal matrix] 
\label{def:transitive}
Define the transitive causal matrix $\calT^*(\calG)$ as:
    \begin{eqnarray*} \label{equ:T-target}
    \small
    \calT^*(\calG)_{ij} \in \mathbb{R}^{|\calV| \times |\calV|} = 
    \begin{cases}
       1  & {\pi(i)<\pi(j)} \text{ for all $\pi$ consistent with $\calG$} \\
       0  & {\pi(i)>\pi(j)} \text{ for all $\pi$ consistent with $\calG$} \\
       0.5 & \text{Otherwise}
    \end{cases} 
\end{eqnarray*}
\end{definition}

We can see that when the causal order of nodes $i$ and $j$ is interchangeable, instead of randomly setting their orders either as $i\to j$ or $j \to i$, we deterministically set their causal relation $ \calT^*(\calG)_{ij}= \calT^*(\calG)_{ji}=0.5$ symmetrically. 

Then we define our $CD$ measure, which is the difference between the overlapping parts of the transitive causal matrices of two graphs: 
\begin{definition}[Causal Difference] \label{def:causal-difference}
    Define the \textbf{C}ausal \textbf{D}ifference 
    between $\calG_u, \calG_v$ as $CD(\calG_u, \calG_v)$ with following formula: 
    \begin{equation} \label{equ:causal-difference}
     CD(\calG_u, \calG_v) = \sum_{i \in \calV_u \cap \calV_v} \sum_{j \in \calV_u \cap \calV_v} (\calT^*(\calG_u)_{ij} - \calT^*(\calG_v)_{ij})^2.
    \end{equation}
\end{definition}
By Definitions \ref{def:transitive} and \ref{def:causal-difference}, we can see $ CD(\calG_u, \calG_v)$ describes the transitive causal difference between DAGs better. 

\begin{figure}[t]
    \centering
    \includegraphics[width=0.8\columnwidth]{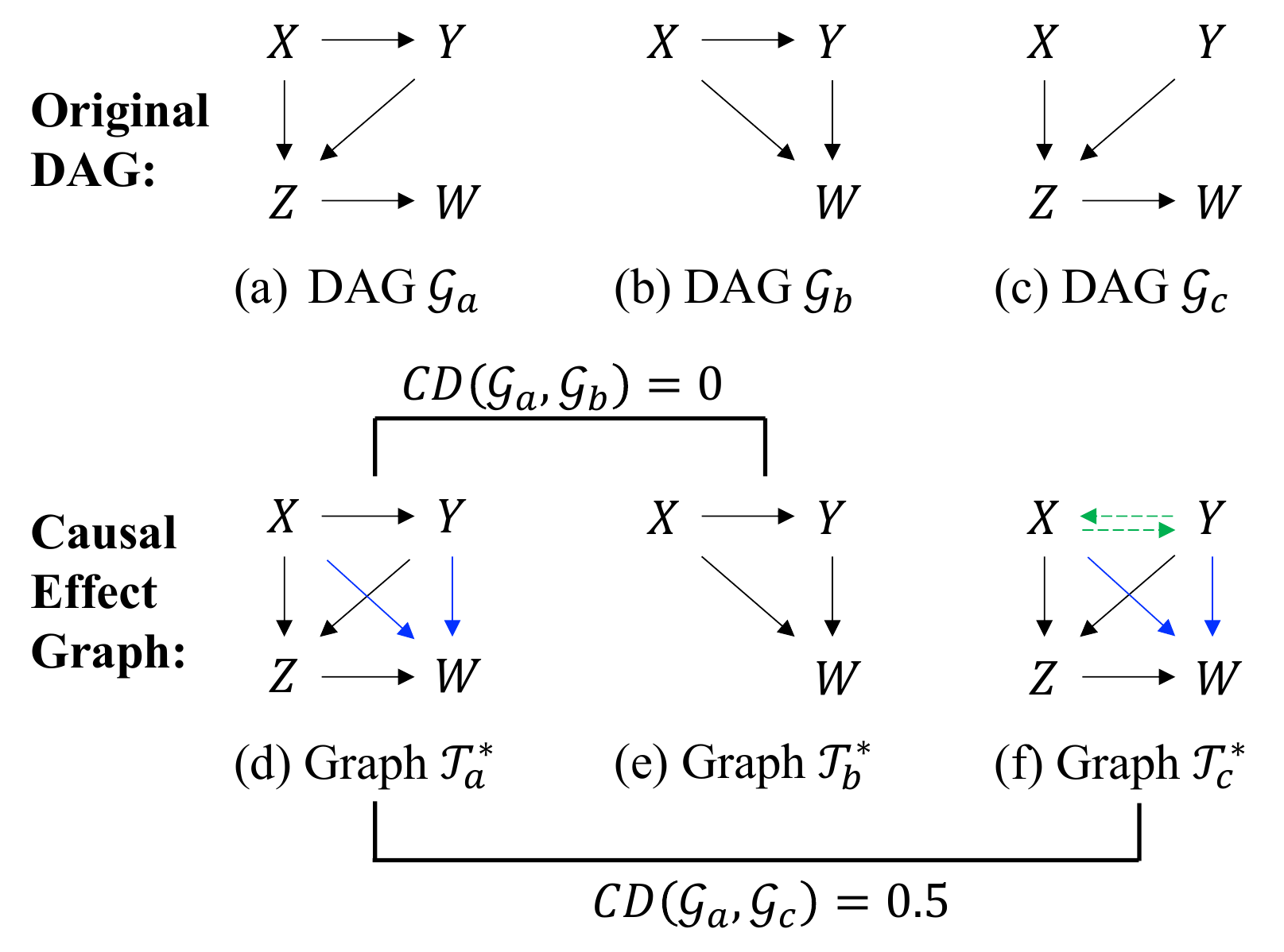}
    \caption{Illustration of the transitive causal matrix. A blue directed edge $i \to j$ represents an added $\calT^*_{ij}=1$. A green directed edges $i \to j$ represents an added $\calT^*_{ij}=0.5$.}
    \label{fig:design}
\end{figure}

Fig. \ref{fig:design} illustrates our design of $\calT^*_a, \calT^*_b$ and $\calT^*_c$, which can be viewed as ``fully-connected'' versions of $\calG_a, \calG_b$ and $\calG_c$. We obtain the causal effect behind the graph $\calG$ and obtain graph $\calT^*$. From Fig. \ref{fig:design}, we show that the edges $X \rightarrow W$ and $Y \rightarrow W$ appear in both $\calT^*_a$ and $\calT^*_b$, which indicates $CD(\calG_a, \calG_b)=0$. Meanwhile, in $\calT_a^*$, the edge $X \rightarrow Y$ is directed with weight 1, but in graph $\calT_c^*$, this edge is bi-directed with weight 0.5. From Eq. (\ref{equ:causal-difference}), $CD(\calG_a, \calG_c) = 0.5^2 + 0.5^2 = 0.5$.

\textbf{Topological Interpretation:} 
We further give a topological interpretation of $\calT^*$ and $ CD(\calG_u, \calG_v)$. By showing that $\calT^*$ lies in a $T_0$ space (or Kolmogorov space \cite{hocking1988topology}), we prove that $CD$ is equivalently defined by the projection and a distance metric of $T_0$ spaces.

\begin{definition}
Define $\calK_\calV$ is the set of $\calT^*$ matrix generated by node set $\calV$, i.e., $\calK_\calV=\{\calT^*(\calG)|\calG=(\calV,\calE), \forall \calE\}$.
\end{definition}

\begin{lemma} \label{the:T0-space}
 $\calK_{\calV}$ is a finite $T_0$ space with $|\calK_\calV|=\alpha(|\calV|)$, where $\alpha(n)$ is the number of distinct $T_0$ topologies with $n$ points, and $\calT^*(\calG) \in \calK_\calV$ corresponds to a unique $T_0$ topology.
\end{lemma}
\begin{proof}
Since $\calT^*(\calG)$ is a bi-directed transitive causality on the set $\calV$, Lemma \ref{the:T0-space} is proved by \cite{finch2003transitive}. 
\end{proof}
\begin{definition} \label{def:T0-distance}
    Define $D_{\calK_\calV}(\calT^*(\calG_1),\calT^*(\calG_2))=\|\calT^*(\calG_1)-\calT^*(\calG_2)\|_F^2, \forall \calT^*(\calG_1),\calT^*(\calG_2) \in \calK_\calV$, which is a distance metric of space $\calK_\calV$. 
\end{definition}
\begin{definition} \label{def:T0-projection}
    Define the projection function $f_{\calV,\calV'}:\calK_\calV \rightarrow \calK_{\calV'}$ as $f(\calT^*(\calG))= \calT^*(\calG')$ where $\calV' \subseteq \calV, \calG' = (\calV',\calE')$ with $\calE'=\{(i,j) | (i,j) \in \calE, i,j \in \calV' \}$.
\end{definition}

\begin{theorem}[\textbf{Topological interpretation}] \label{the:CD-represent}
    The causal difference $CD$ in Eq. (\ref{equ:causal-difference}) can be represented by the following formula:
    \begin{equation}
        CD(\calG_u,\calG_v) = D_{\calK_\calV}(f_{\calV_u,\calV}(\calT^*(\calG_u)), f_{\calV_v,\calV}(\calT^*(\calG_v)))
    \end{equation}
    where $\calG_u=(\calV_u, \calE_u)$, $\calG_v=(\calV_v, \calE_v)$, and $\calV=\calV_u \cap \calV_v$, $D_{\calK_\calV}$ means the distance metric in space $\calK_\calV$.
\end{theorem}
\begin{proof}
\small
\begin{align*}
&D(f_{\calV_u,\calV}(\calT^*(\calG_u)), f_{\calV_v,\calV}(\calT^*(\calG_v))) \\
&= \| f_{\calV_u,\calV}(\calT^*(\calG_u)) - f_{\calV_v,\calV}(\calT^*(\calG_v)) \|_F^2 \\
&= \sum_{i \in \calV} \sum_{j \in \calV} (f_{\calV_v,\calV}(\calT^*(\calG_v))_{ij} - f_{\calV_v,\calV}(\calT^*(\calG_v))_{ij})^2 \\
&= \sum_{i \in \calV_u, \calV_v} \sum_{j \in \calV_u, \calV_v} (\calT^*(\calG_u)_{ij} - \calT^*(\calG_v)_{ij})^2 \\
&= CD(\calG_u, \calG_v)
\end{align*}
\end{proof}

\begin{figure}[t]
    \centering
    \includegraphics[width=\columnwidth]{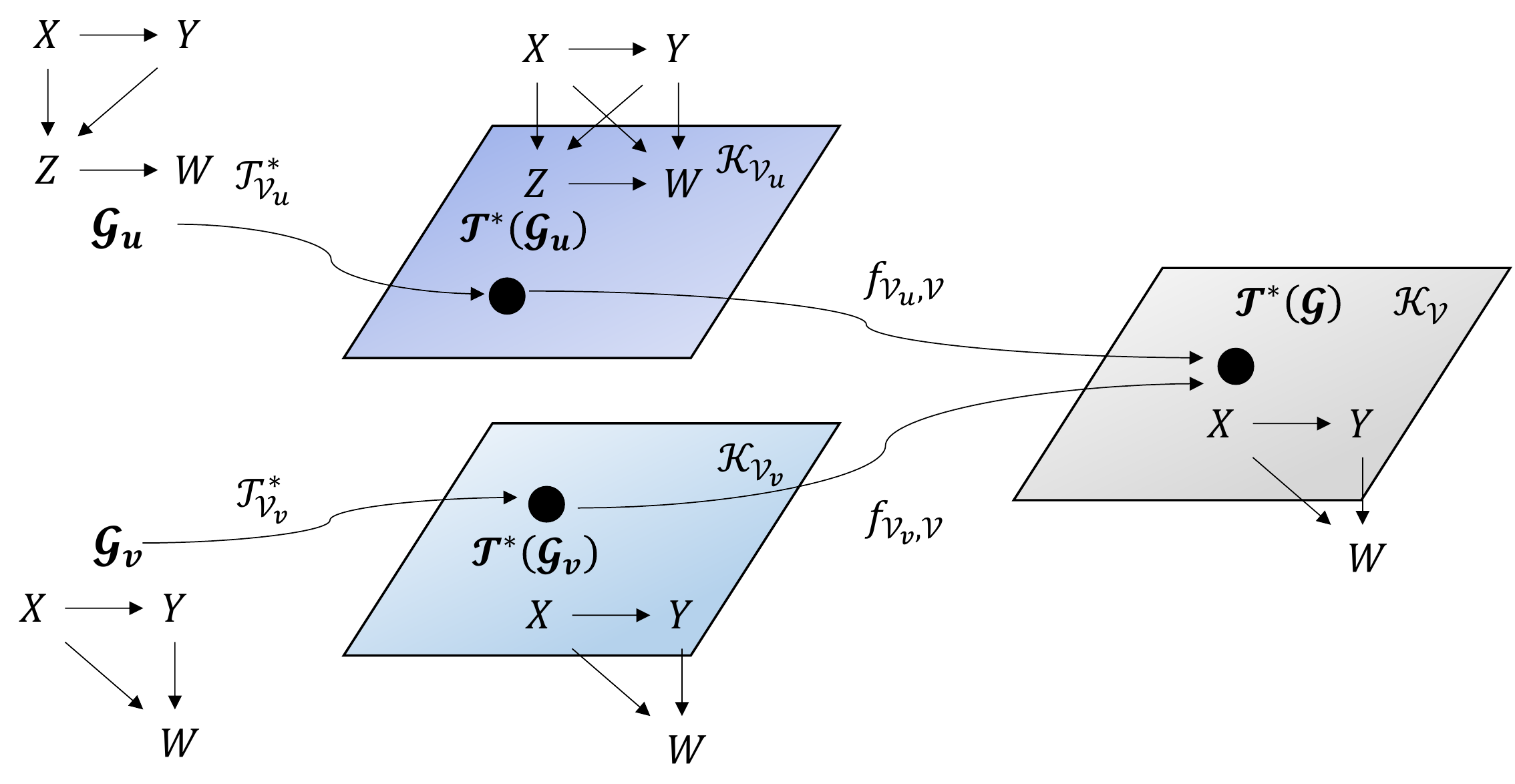}
    \caption{The illustration of our design from a topological perspective. In this case, $CD(\calG_u,\calG_v)=0$ since they correspond to the same $\calT^*(\calG)$ in space $\calK_\calV$.}
    \label{fig:space}
\end{figure}

\noindent Lemma \ref{the:T0-space} shows that our design $\calT^*(\calG)$ lies in a $T_0$ space $\calK_\calV$. Using Def. \ref{def:T0-distance} and \ref{def:T0-projection}, we define the distance metric in $T_0$ space and the projection function of two $T_0$ spaces. Finally, Theorem \ref{the:CD-represent} shows that our difference measure $CD(\calG_u,\calG_v)$ can be represented by the distance in space $\calK_\calV$, where $\calV=\calV_u \cap \calV_v$, as shown in Fig.~\ref{fig:space}.

\textbf{Continuous Trick $\calT$:} Although $\calT^*$ and $CD(\calG_u,\calG_v)$ have such good properties, they are incompatible with the current score-based algorithm in Eq. (\ref{equ:multitask-scorebased}) since $\calT^*$ is discrete thus without gradient. To still guarantee our structure learning algorithm can be solved with gradient-based methods, we further derive a differentiable design $\calT$ as an approximation of $\calT^*$ in Def.~\ref{def:continuous-T}, and also prove the consistency for the conversion.

 \begin{definition}[Differentiable Transitive Causal Matrix] \label{def:continuous-T}
Define the differentiable transitive causal matrix as
     \begin{equation} \label{equ:T-design}
        \calT(\bW) = {\rm S} \left(c(l(\bW) - l(\bW)^T)\right), \text{where } l(\bW) = \bI + \sum_{i=1}^{P} \bW^i.
    \end{equation}
$c$ is a positive constant, $\bW$ is the  adjacency matrix of graph $\calG=(\calV,\calE)$, $\bW_{ij} > 0$ means $(i,j) \in \calE$, function $S$ is the element-wise Sigmoid function for matrix $\bX$ with $S(\bX)_{ij} = \frac{1}{1 + \exp(-\bX_{ij})}$. 
 \end{definition}
 \begin{theorem}[\textbf{Consistency}] \label{the:consistency}
     If~~$\bW$ is the adjacency matrix of graph $\calG=(\calV,\calE)$. The differentiable transitive relation matrix $\calT(\bW)$ converges to the transitive relation matrix $\calT^*(\calG)$ as $c \rightarrow \infty$.
 \end{theorem}
 \begin{proof}
     In Eq. (\ref{equ:T-design}), $l(\bW)=\bI+\bW+\bW^2+\cdots$, where the entries of matrix power $\bW^k_{ij}$ are the sum of the weight products along all $k$-step paths from node $i$ to node $j$. Therefore, $l(\bW)_{ij}=0$ means that node $j$ cannot be reached from node $i$ in graph $\bW$, $l(\bW)_{ij} > 0$ means that node $j$ can be reached from node $i$ in graph $\bW$. Since $\bW$ is acyclic, $l(\bW)_{ij}$ and $l(\bW)_{ji}$
     have three cases:
     
\noindent (1) $l(\bW)_{ij}>0, l(\bW)_{ji}=0$, representing the case $\pi(i)<\pi(j)$: 
    \begin{equation*}
        \lim_{c \rightarrow \infty}\calT(\bW)_{ij} = \lim_{c \rightarrow \infty}{\rm Sigmoid} \left(cl(\bW)_{ij}\right) = 1 = \calT^*(\calG)_{ij}
    \end{equation*}
(2) $l(\bW)_{ij}=0, l(\bW)_{ji}>0$, representing the case $\pi(i)>\pi(j)$:
    \begin{equation*}
        \lim_{c \rightarrow \infty}\calT(\bW)_{ij} = \lim_{c \rightarrow \infty}{\rm Sigmoid} \left(-cl(\bW)_{ji}\right) = 0 = \calT^*(\calG)_{ij}
    \end{equation*}
    
\noindent (3) $l(\bW)_{ij}=0, l(\bW)_{ji}=0$, representing the case that the relationship between $\pi(i)$ and $\pi(j)$ is not sure, and $\calT(\bW)_{ij}$ = ${\rm Sigmoid}(0)=\calT^*(\calG)_{ij}=0.5$.

Combining cases (1) to (3):
\begin{equation*}
    \lim_{c \rightarrow \infty} \calT(\bW) = \calT^*(\calG)
\end{equation*}
 \end{proof}
 Theorem \ref{the:consistency} proves the consistency of $\calT$ and $\calT^*$ when $c\rightarrow \infty$. In the algorithm, $c$ can be set to a relatively large constant and avoid floating point overflow. Therefore, the \textbf{D}ifferentiable \textbf{C}ausal \textbf{D}ifference $DCD$ is given by:
\begin{equation} \label{equ:continuous-CD}
     DCD(\bW_u, \bW_v) = \sum_{i \in \calV_u, \calV_v} \sum_{j \in \calV_u, \calV_v} (\calT(\bW_u)_{ij} - \calT(\bW_v)_{ij})^2,
\end{equation}
which is used in our multi-task score-based algorithm in Eq. (\ref{equ:multitask-scorebased}).

\subsection{Structural Learning Algorithm}
\label{sec:algorithm}
To solve Eq. (\ref{equ:multitask-scorebased}), following the algorithm proposed by \cite{zheng2018dags}, we derive a structural learning algorithm based on the Lagrangian method with a quadratic penalty, which converts the score-based method in Eq. (\ref{equ:multitask-scorebased}) to an unconstrained problem:
\begin{equation} \label{equ:Lagrangian dual problem}
\small
    \begin{split}
        F(\bC_{(1)},...,\bC_{(L)}) = &\underset{\bC_{(1)},...,\bC_{(L)}}{\min} \max_{\beta>0}~f(\bC_{(1)},...,\bC_{(L)}) \\
        & \quad + \sum_{l=1}^L \beta h(\bW_{(l)}) + \frac{\alpha^2}{2} h(\bW_{(l)})^2, \\
    \end{split}
\end{equation}
where
\begin{equation}
\small
    \begin{split}
       f &= \sum_{l=1}^L \frac{1}{2N_l} \|\bA_{(l)} - \bA_{(l)}\bC_{(l)}\|_F^2 + \rho \sum_{l_1,l_2} s_{l_1,l_2}DCD(\bW_{(l_1)}, \bW_{(l_2)})  \\
       & \quad + \lambda \sum_{l=1}^L \|\bC_{(l)}\|_1.
    \end{split}
\end{equation}

    
    
    
\noindent $\beta$ is dual variable, $\alpha$ is the coefficient for quadratic penalty. We solve the dual problem by iteratively updating $f(\bC_{(1)},...,\bC_{(L)})$ and $\beta$. Due to the smoothness of objective 
$F$, Adam \cite{kingma2014adam} is employed to minimize $F$
given $\beta$, and update $\beta$ by $\beta \leftarrow \beta + \alpha \sum_{l=1}^L h(\bW_{(l)})$. The overall steps are summarized in Algorithm \ref{alg:Multi-task learning algorithm} and the convergence property of our algorithm is fully discussed by \citet{ng2022convergence}. The partial derivative function of $F$ for $\bC_{(1)}, \ldots, \bC_{(L)}$ are computed by the following three parts:

\noindent (1) Derivative of $\|\bA_{(l)} - \bA_{(l)}\bC_{(l)} \|_F^2$:
   \begin{equation*}
        \frac{\partial \|\bA_{(l)} - \bA_{(l)}\bC_{(l)} \|_F^2}{\partial \bC_{(l)}} = -2 \bA^T_{(l)} (\bA_{(l)} - \bA_{(l)}\bC_{(l)}).
    \end{equation*}
    
\noindent (2) Derivative of $h(\bW_{(l)})$:
    \begin{equation*}
    \begin{split}
        \frac{\partial h(\bW_{(l)})}{\partial \bC_{(l)ij}} &= \frac{\partial h(\bW_{(l)})}{\partial \bW_{(l)}} \frac{\partial \bW_{(l)}}{\partial \bC_{(l)ij}}.
    \end{split}
    \end{equation*}
    where $\frac{\partial h(\bW_{(l)})}{\partial \bW_{(l)}} = e^{\bW_{(l)}}$, $\frac{\partial \bW_{(l)}}{\partial \bC_{(l)ij}}$ can be obtained from the definition $\bW_{(l)ij}=\|\bC_{(l)ij}\|_F^2$.
    
    \vspace{5pt}
\noindent (3) Derivative of $DCD(\bW_{(l_1)}, \bW_{(l_2)})$:
    \begin{equation}
    \small
    \begin{aligned}
        &\frac{\partial DCD(\bW_{(l_1)}, \bW_{(l_2)})}{\partial \bC_{(l)ij}} = \frac{\partial DCD}{\partial l(\bW_{(l_1)})} \frac{\partial l(\bW_{(l_1)})}{\partial \bW_{(l_1)}} \frac{\partial \bW_{(l_1)}}{\partial \bC_{(l_1)ij}}\\\vspace{5pt}
        &\frac{\partial DCD}{\partial l(\bW_{(l_1)})_{ij}} =  2Q\left(cl(\bW_{(l_1)})_{ij}-cl(\bW_{(l_1)})_{ji}, cl(\bW_{(l_2)})_{ij}-cl(\bW_{(l_2)})_{ji} \right)\\\vspace{5pt}
        & \textrm{where  } Q(x,y) = \frac{2c e^x (e^x - e^y)}{(1 + e^x)^3 (1 + e^y)}\\
        &\frac{\partial l(\bW_{(l_1)})_{kl}}{\partial \bW_{(l_1)ij}} = \sum_{p=1}^P \sum_{r=1}^p (\bW_{(l_1)}^{r} \bJ_{ij} \bW_{(l_1)}^{(p-r-1)})_{kl}
    \end{aligned}
    \end{equation}
    where $\bJ_{ij}$ is a $P_{l_1} \times P_{l_1}$ matrix with $(\bJ_{ij})_{ij}=1$ and 0 in other entries. Denote $P=\max P_l$ and $M=\max M_l$, in each Adam iteration, the overall computation complexity is $O(LNM+L^2P^2M^2+LP^6)$. The detailed math is in Appx. \ref{app:complexity}.

\begin{algorithm}[t]
\caption{Multi-task learning algorithm}
\label{alg:Multi-task learning algorithm}
\begin{algorithmic}
   \STATE {\bfseries Input:} Number of task $L$, multi-task data $\bX_{(1)}, \ldots, \bX_{(L)}$, similarity $s_{ij}$, learning rate $r$, tolerance $h_{\min}$
   \STATE Obtain $\bA_{(l)}$ from $\bX_{(l)}$ via Eq. (\ref{equ:A})
   \STATE Initialize $\bC_{(1)},\bC_{(2)},...,\bC_{(L)}, \alpha \leftarrow 1, \beta \leftarrow 0$
   \REPEAT
   \STATE  $\bC_{(1)},\bC_{(2)},...,\bC_{(L)} \leftarrow \text{Adam}(\bA_{(1)},\bA_{(2)},...,\bA_{(L)})$. \# Use Adam to solve Lagrangian dual problem in Eq.  (\ref{equ:Lagrangian dual problem})
   \STATE $\beta \leftarrow \beta + \alpha \sum_{l=1}^L h(\bW_{(l)})$
   \STATE $\alpha \leftarrow r\alpha$
   \UNTIL{$\sum_{l=1}^L h(\bW_{(l)}) < h_{\min}$}
\end{algorithmic}
\end{algorithm}

\subsection{Extension to nonlinear cases}
\label{sec: non-linear}
Our model can be extended to nonlinear models with ease. To model a nonlinear system, we would need to design two components. Firstly, we develop the transition function in DAG, which can be expressed as $\mathbb{E}(X_j|X_{pa_j})=g_j(f_j(X))$, where $f_j:\mathbb{R}^T\rightarrow\mathbb{R}$ and $g_j:\mathbb{R}\rightarrow\mathbb{R}$. In our design, we utilized mo2mo regression to construct $f_j$ and $g_j$. However, these functions can also be constructed using kernel or deep methods such as graph neural network \citep{yu2019dag}. Secondly, we would need to construct an adjacency matrix of the causal graph, W, that satisfies the condition $f_{ij}\ne 0 \rightarrow W_{ij}>0$. An easy way to achieve this is by setting $W_{ij}=||f_{ij}||_{L^2}$. Then the objective loss function can be constructed and the Notears constraints can be added. By following this procedure, our multitask design with the CD constraints can also be added. Consequently, our multi-task learning framework can be easily extended to nonlinear models.

\section{Experimental study}
\label{sec:experiment}
\subsection{Synthetic Data}

\begin{table}[t]
\small
\centering
\caption{\label{tab:numerical-task-setting} The task settings}
\vspace{-10pt}
\begin{tabular}{ccc}
\toprule
Model & Multi-modal design & Multi-task method \\
\midrule
MM-DAG & Yes & Causal Difference \\
Separate & Yes & None \\
Matrix-Difference & Yes & Matrix Difference \\
Order-Consistency & Yes & Order Consistency \\
MV-DAG & No & Causal Difference \\
\bottomrule
\end{tabular}
\end{table}

The set of experiments is designed to demonstrate the effectiveness of MM-DAG. We first generate a ``full DAG'' $\calG_0$ with $P$ nodes: $\calG_0 \in \{0, 1\}^{P \times P}$ with $[\calG_0]_{ij} = \mathbb{I}([e^{\calG}]_{ij}>0)$, where $\calG$ is generated by Erdös-Rényi random graph model \cite{erdos1960evolution}. Nodes $1, 2, \ldots, \lfloor P/2\rfloor$ are set as scalar variables, and nodes $\lfloor P/2+1\rfloor, \ldots, P$ are functional variables with the same $K$ Fourier bases $\nu_1(t), \ldots, \nu_K(t)$. Therefore, the node variables can be represented by: 
\begin{eqnarray}
    \bX^{(n)}_j = 
    \begin{cases}
       \ba ^{(n)}_j  & {\rm for \ scalar \ nodes} \\
       \sum_{k=1}^K a_{jk}^{(n)} \nu_k(t)  & {\rm for \ functional \ nodes} 
    \end{cases} ,
\end{eqnarray}
{Then, we sample $L$ sub-graphs from $\calG_0$ as different tasks. For task $l$ with $\calG_{(l)}$, we randomly select node set $\calV_{(l)} \subseteq \{1,2,...,P\}$ with $P/2 \leq |\calV_{(l)}| \leq P$, and denote node $i$ in task $l$ as $node[l,i]$, that is, $\{node[l,i]|i\in\calG_{(l)}\} = \calV_{(l)}$. }
To generate $N$ samples for $L$ tasks, we first generate $\bC_{(l)ij} = c_{(l)ij} * \bI_{(l)ij} * [\calG_0]_{node[l,i], node[l,j]}$, where $c_{(l)ij}$ is sampled from the uniform distribution $\mathcal{U}(-2, -0.5) \cup (0.5,2)$ and $\bI_{(l)ij}$ is a matrix whose diagonal is 1, with the dimension the same as $\bC_{(l)ij}$. Thus, we ensure the causal consistency of tasks $\calG_{(l)}$ generated by $\calG_0$. Then we generate $\ba_{j}^{(n)}$ according to Eq. (\ref{equ:A}) with $\tilde{\be}_{j}^{(n)}\sim N(\mathbf{0},\mathbf{I})$. 

For evaluation, F1-score (F1), false positive rate (FPR), and true positive rate (TPR) \cite{raschka2014overview} are employed as the quantitative metrics. Higher F1 ($\uparrow$) and TPR ($\uparrow$) indicate better performance, whereas FPR is reversed ($\downarrow$). 
We compare our MM-DAG with four baselines. (1) \textbf{Separate} based on NoTears \cite{zheng2018dags} is to learn the multi-modal DAG for each task separately by optimizing Eqs. (\ref{equ:score-based}) and (\ref{equ:acylic}).  (2) \textbf{Matrix-Difference} is to use the matrix distance $\Delta$ as the difference measure in the multi-task learning algorithm, which has limitations to handle \textbf{Case I}. (3) \textbf{Order-Consistency} is the multi-task causal graph learning of \cite{chen2021multi}, which assumes all the tasks have the same causal order.  It has limitations in dealing with \textbf{Case II} (See Fig.~\ref{fig:situation}). (4) \textbf{MV-DAG}: Instead of mo2mo regression, MV-DAG implements a preprocessing method for functional data by dividing the entire time length of the function into 10 intervals and averaging each interval. This transforms each functional data into a ten-dimensional vector. We show the difference between the five models in Table \ref{tab:numerical-task-setting}.

\begin{figure}[t]
    \centering
    \includegraphics[width=\columnwidth]{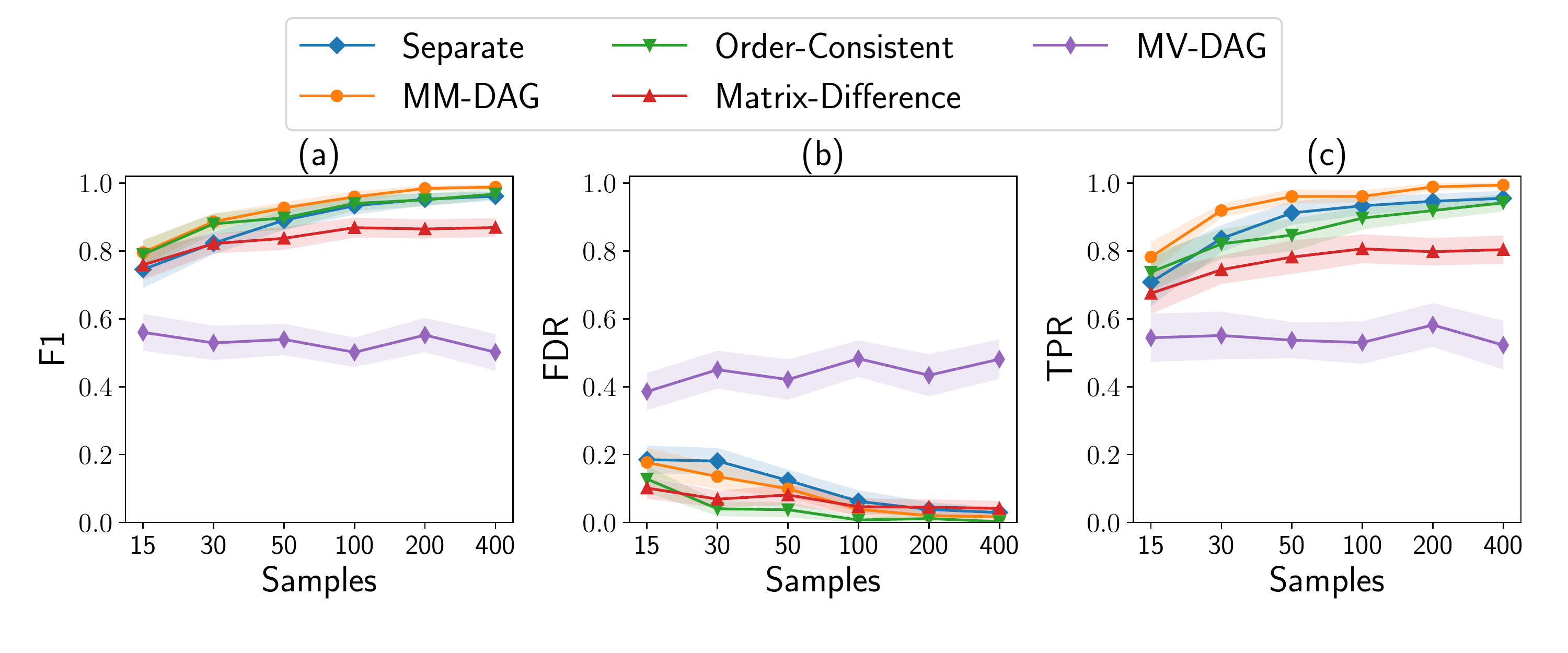}
    \caption{The F1 score ($\uparrow$), FPR ($\downarrow$), TPR ($\uparrow$) across the sample sizes (task number $L=4$), under 95\% confidence interval.}
    \label{fig:result-samples}
\end{figure}

\begin{figure}[t]
    \centering
    \includegraphics[width=\columnwidth]{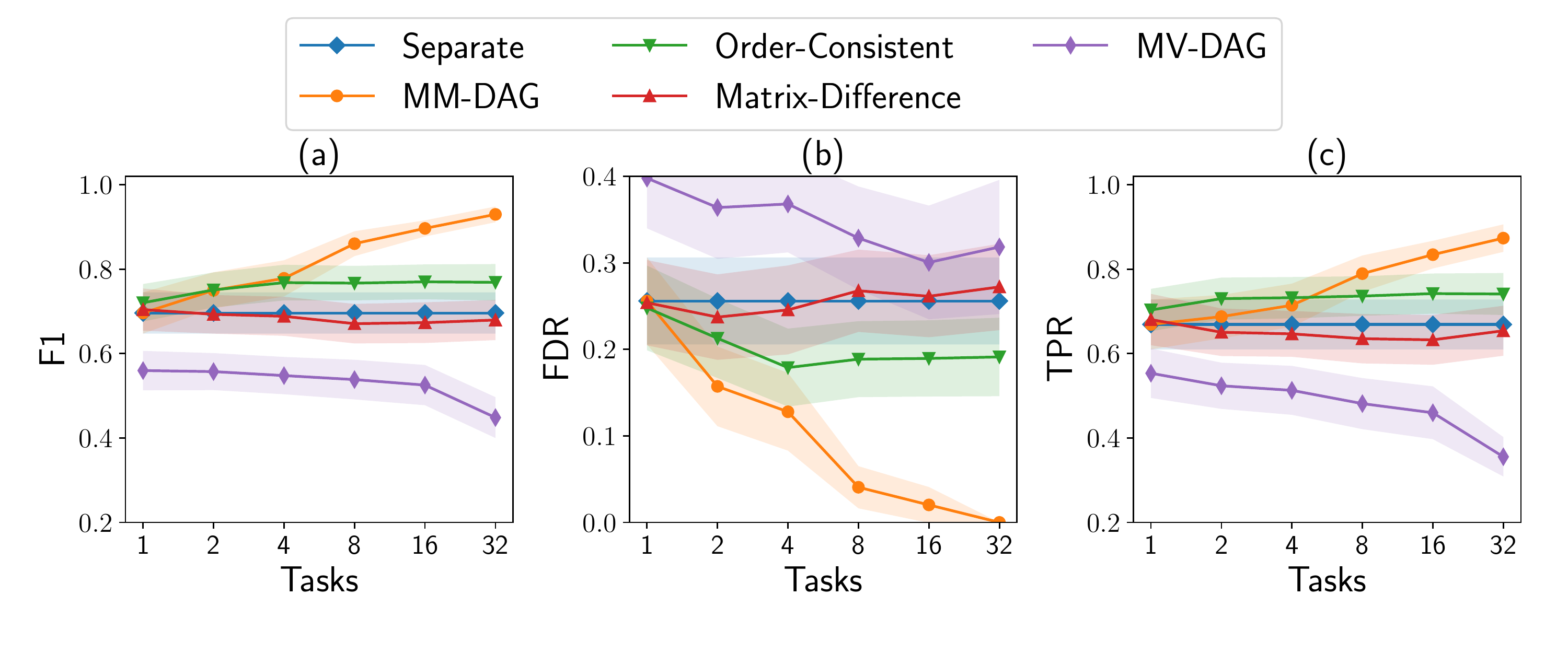}
    \caption{The F1 score ($\uparrow$), FPR ($\downarrow$), TPR ($\uparrow$) across the number of tasks (sample size $N=10$), under 95\% confidence interval.} 
    \label{fig:result-design}
\end{figure}

\textbf{Relationship between model performance and sample size:} 
We first fix the number of tasks (DAGs) $L=4$. The evaluation metrics are shown in Fig.~\ref{fig:result-samples} under different sample sizes. We see that MM-DAG outperforms the baselines, with the highest F1 score and a \textbf{+2.95\%} gain than the best peers, i.e., \textit{order-consistency}, when $N=50$, and \textbf{+11.9\%} gain than \textit{Matrix-Difference} when $N=200, 400$. The performance of the four methods improves as the number of samples $N$ increases. Notably, we discover that the F1 score of baseline \textit{Matrix-Difference} is stable at 0.83 even if we increase $N$ from 100 to 400. This is attributed to biased estimates caused by the fact that the matrix difference incorrectly penalizes the correct causal structure of the task. This bias cannot be reduced by increasing the number of samples. Thus, the F1 score of \textit{Matrix-Difference} cannot reach 100\%. By comparing our proposed MM-DAG model to the MV-DAG model, we verify the contribution of the multi-modal design.

\textbf{Relationship between model performance and the number of tasks:}
We set the sample size $N=10$, and investigate the effect of task size $L$. The results are shown in Fig.~\ref{fig:result-design}, from which the salient benefits of our proposed MM-DAG can be concluded: As the number of tasks increases, the performance of our method improves the fastest, but the baseline \textit{Separate} holds. Promisingly, our method gains a maximum \textbf{+16.1\%} gain on F1 against its best peers, i.e., \textit{order-consistency}, when $L=32$. It can successfully exploit more information in multi-task learning since it can better deal with the uncertainty of causal orders. All of those demonstrate the superiority of MM-DAG. 

\textbf{Visualization:} We also visualize the learned DAGs in Fig. \ref{fig:heatmap-tasks}, which shows the estimated adjacent matrix (edge weights) $\bW_{l}$ of {MM-DAG}, \textit{Order-Consistency}, and \textit{Matrix-Difference}.
MM-DAG derives the most accurate results, which indicates that it achieves the best performance in estimating the DAG structure. 

Appx. \ref{app:comparison} shows the detailed experiments result in the case that $N=20, L=10$ and shows that our MM-DAG have best F1 score. 

\textbf{Ablation study of CD penalty and L1-norm penalty:} We conduct an ablation study with number of nodes $N=20$ and number of tasks $L=10$. The results of MM-DAG$(\lambda=0.01,\rho=0.1)$, MM-DAG$(\lambda=0,\rho=0.1)$ and MM-DAG$(\lambda=0.01,\rho=0)$ are shown as follows. The result shows that the $L_1$ penalty imposed by $\lambda$ can reduce overfitting and reduce FDR. The causal difference penalty imposed by $\rho$ can combine task information to decrease FDR and increase TPR.

\begin{table}[th]
\small
\centering
\caption{\label{tab:ablation-study} The ablation study result ($N=20, L=10$)}
\vspace{-10pt}
\begin{tabular}{ccccc}
\toprule
$\lambda$ & $\rho$ & ${\rm FDR}(\pm \rm std)$ &	$\rm TPR(\pm \rm std)$ & $\rm F1(\pm \rm std)$\\
\midrule
0.01 & 0.1 & $\mathbf{0.077\pm0.071}$ &	$\mathbf{0.942 \pm0.075}$ &	$\mathbf{0.929\pm0.051}$ \\
0.00 & 0.1 &$0.517\pm0.029$	&$0.961\pm0.051$&	$0.604 \pm 0.029$ \\
0.01 & 0.0 & $0.344\pm0.123$ &	$0.795\pm0.118$	& $0.709\pm0.097$\\
\bottomrule
\end{tabular}
\vspace{-0.5cm}
\end{table}

\begin{figure}[t]
    \centering
    \includegraphics[width=0.95\columnwidth]{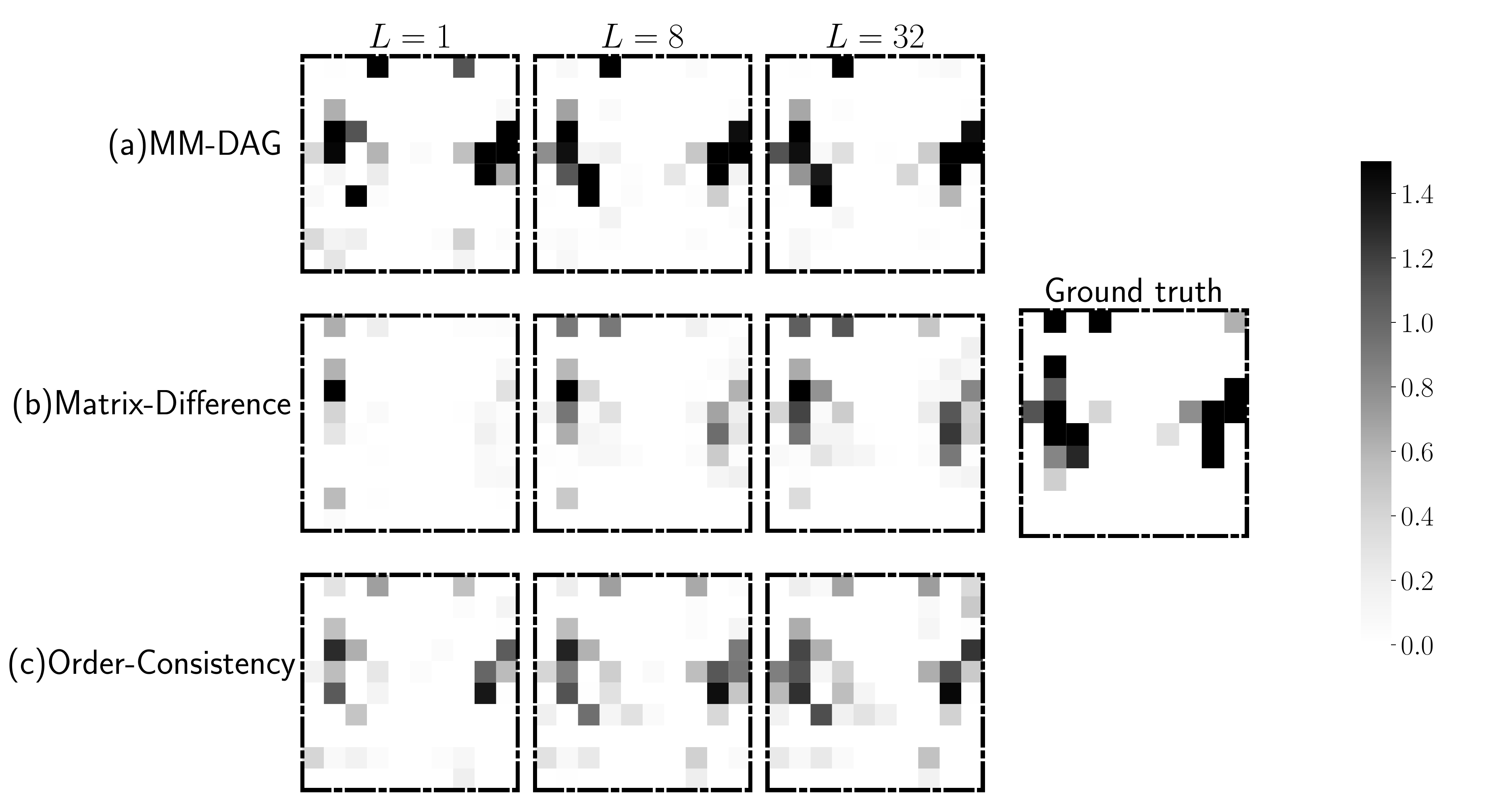}
    \caption{The estimated adjacent matrix $\bW\in \mathbb{R}^{10\times 10}$ of task 1 by three multitask methods with different task numbers $L$. }
    \label{fig:heatmap-tasks}
\end{figure}


\subsection{Congestion Root Causes Analysis}
For the traffic scenario application, we apply our method to analyze the real-world congestion causes of five intersections in FenglinXi Road, Shaoxing, China, including four traffic light-controlled intersections and one traffic light-free intersection, as shown in Fig. \ref{fig:newsumoresult} in Appendix \ref{app:newsumoscenario}). The original flow is taken from the peak hour around 9 AM. We reconstruct the exact flow given our real data. According to reality, the scenario is reproduced in the simulation of urban mobility (SUMO) \cite{behrisch2011sumo}.  There are three types of variables in our case study, as summarized in Table \ref{tab:Abbreviation-func}: (1) The scalar variables $X$, such as Origin-Destination (OD) or intersection turning probability, represent the settings of SUMO environment and can be adjusted. (2) The functional variables $Y(t)$ represent the traffic condition variables such as mean speed or occupation. (3) The vector variables $R$ represent the congestion root cause. Since these types of variables are obtained at a lower frequency compared with the traffic condition variables, we can regard them as vectors.  For each sample, we set different levels on each variable of $X$, then $Y$s are collected by the sensor in SUMO, and $R$s are obtained with rule-based algorithms. 
\begin{table}[t]
\small
\centering
\caption{Node descriptions}
\vspace{-10pt}
\begin{tabular}{cccc}
\toprule
Node & Abbreviation & Description & Type \\
\midrule
$X_1$ & OD & OD demand & \multirow{3}*{Scalar}\\
$X_2$ & Turn & Turning probability & \\
$X_3$ & CT & Cycle time of traffic light & \\
\midrule
$R_1 \in \mathbb{R}^5$ & Cycle-L & Long cycle time of traffic light & \multirow{5}*{Vector} \\
$R_2 \in \mathbb{R}^5$ & Cycle-S & Short cycle time of traffic light & \\
$R_3 \in \mathbb{R}^2$ & Lanes-irr & Irrational guidance lane & \\
$R_4 \in \mathbb{R}^5$ & Phase-irr & Irrational phase sequence & \\
$R_5 \in \mathbb{R}^4$ & Congest & Congestion & \\
\midrule
$Y_1(t)$ & OC & Occupancy of the intersection  & \multirow{2}*{Function} \\
$Y_2(t)$ & MS & Mean speed of the intersection & \\
\bottomrule
\end{tabular}
\label{tab:Abbreviation-func}
\end{table}

The characteristics of five intersections are summarized in Table \ref{tab:task-setting}. The second intersection has no traffic light; thus, it has only six nodes without traffic light-related variables $X_3, R_1, R_2, R_4$. Since the number of lanes is different, the number of $X$ varies across tasks, which leads to a different number of samples. 


\begin{table}[t]
\small
\centering
\caption{\label{tab:task-setting} The task settings}
\vspace{-10pt}
\begin{tabular}{ccccc}
\toprule
Task & Lanes & Has traffic light? & Nodes & Samples \\
\midrule
1 & 15 & Yes & 10 & 1584\\
2 & 6 & No & 6 & 324\\
3 & 12 & Yes & 10 & 1296\\
4 & 6 & Yes & 10 & 972\\
5 & 13 & Yes & 10 & 1296\\
\bottomrule
\end{tabular}
\end{table}


Practically, traffic setting variables affect the congestion situations, and the different types of congestion can lead to changes in traffic condition variables. Therefore, it is assumed that there is only a one-way connection from $X$ to $R $ and from $R $ to $Y$ (This hierarchical order, i.e., scaler $\to$ vector $\to$ functional data, is only specific to this domain, which should not be generalized to other domains). Furthermore, considering the setting variables are almost independent, there are no internal edges between different $X$s. Additionally, we assume that some congestion causes may produce others, which should be concerned about. To this end, the interior edges in $R$ are retained when estimating its causal structure. 

\begin{figure}[t]
    \centering
    \includegraphics[width=0.99\columnwidth]{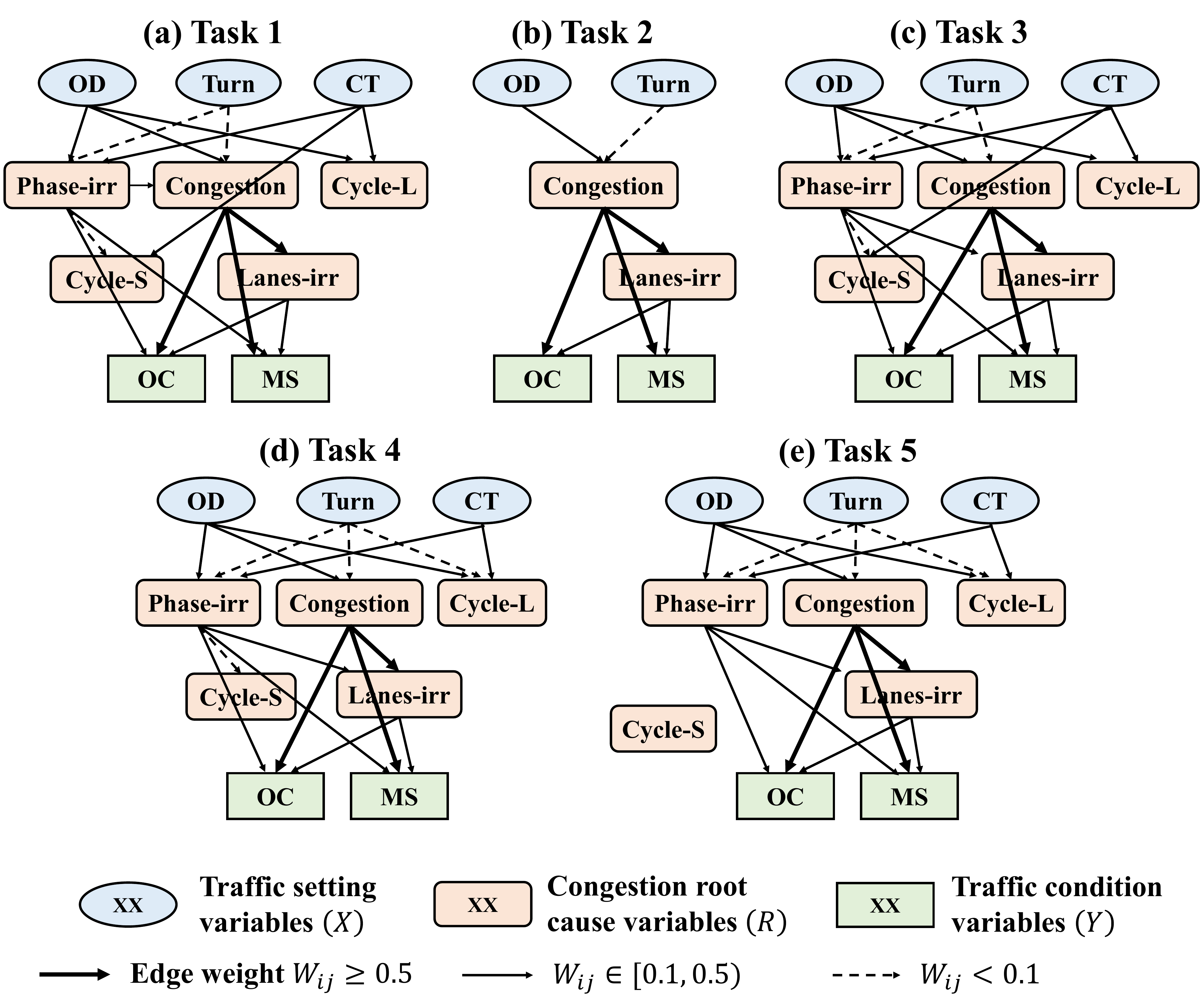}
    \caption{The hierarchical illustrations of inferred DAGs for the traffic application. Each intersection is treated as a task.}
    \label{fig:multi-task}
\vspace{-10pt}
\end{figure}

In the multi-task settings, we assign the task similarity $s_{i,j}$ as the inverse of the physical distance between intersection $i$ and intersection $j$. For the functional PCA, the number of principal components is chosen as $K=5$. Fig. \ref{fig:multi-task} shows the results of our multi-task learning algorithm.
One can figure out the results by analyzing the points of commonalities and differences in the 5 tasks. The variables (nodes) in each task (DAG) are divided into three hierarchies, i.e., $X, R $, and $Y $ for better illustration. 

We can find some interesting insights from the results. For the four intersections with traffic lights, the causal relationships are similar to local differences. Generally, for edges from \textbf{$X$ $\rightarrow$ $R$}, changes in \textit{OD demand} affect \textit{traffic congestion}, \textit{irrational phase sequences}, and \textit{long cycle times}. \textit{Turning probability} adjustments can slightly result in \textit{congestion} and \textit{irrational phase sequences} with lower likelihoods, whereas traffic light adjustments may cause \textit{long signals} or \textit{short signal times} and \textit{irrational phase sequences}. For edges from \textbf{$R$ $\rightarrow$ $R$}, we can see both \textit{irrational phase sequences} and \textit{congestion} may lead to an \textit{irrational guidance lane}. For edges from \textbf{$R$ $\rightarrow$ $Y$}, \textit{congestion}, \textit{irrational phase sequences}, and \textit{irrational guidance lanes} can cause high \textit{occupancy} and yet low \textit{speed}. It is to be noted that for tasks 4 and 5, the \textit{cycle time} of the traffic light will not lead to its \textit{short cycle time}. This might be because they are three-way intersections and have \textit{smaller traffic flows}. Consequently, short cycle time may not occur. 
For the traffic light-free intersection Task-4, its causal relations are the same as the overlapping parts of the other four. 

We can draw some primary conclusions from Fig. \ref{fig:multi-task}(a) that: (1) The change of \textit{OD demand} is the most critical cause for traffic congestion, whereas the impact of turning probability on it is slight (edge weight $< 0.1$). (2) \textit{Cycle time} does not directly cause congestion, but sometimes it can produce \textit{irrational phase sequence} and thus cause congestion indirectly.

In Appendix \ref{app:newsumoscenario}, we further test our model when dealing with a more complex and realistic case where all the intersections are connected and interdependent. 

\section{Conclusion}
\label{sec:conclusion}
This paper presents the multi-task learning algorithm for DAG to deal with multi-modal nodes. It first conducts \textit{mulmo2} regression to describe the linear relationship between multi-modal nodes. Then we propose a score-based algorithm for DAG multi-task learning. We propose a new $CD$ function and its differentiable form to measure and penalize the difference in causal relation between two tasks, 
better formulating the cases of unincluded nodes and uncertainty of causal order. We give important theoretical proofs about topological interpretation and the consistency of our design. The experiments show that our MM-DAG can fuse the information of tasks and outperform the separate estimation and other multi-task algorithms without considering the transitive relations. Thus, our design of causal difference has a strong versatility, which can be extended to other types of multi-task DAG in future work, such as federated multi-task DAG learning \cite{gao2021feddag}. It is worth mentioning that we start the multi-task DAG learning for multi-modal data with a linear model first since this field is still unexplored and linear assumption is easy to comprehend. 

\bibliographystyle{ACM-Reference-Format}
\bibliography{ms}


\newpage
\appendix
\noindent \textbf{\Large APPENDIX}
\vspace{8pt}

\noindent This appendix provides additional details on our paper. Appendix \ref{app:complexity} analyzes the complexity of our algorithm for each Adam iteration. Appendix \ref{app:comparison} presents detailed results from our numerical study. Appendix \ref{app:newsumoscenario} introduces a new SUMO scenario and compares the results with those from the old scenario. Appendix \ref{app: future} discusses the potential future work.

\section{Computation of the complexity} \label{app:complexity}
The most computationally heavy part is computing the gradients in Section \ref{sec:algorithm}. We calculate the computation complexity of each iteration in the gradient-based algorithm as follows:

\begin{itemize}
    \item Derivative of $\|\mathbf{A}_{(l)} - \mathbf{A}_{(l)}\mathbf{C}_{(l)} \|_F^2$: the computation complexity is $O(N_lM_l)$. Therefore, for all tasks, the computation complexity is $O(\sum N_lM_l)$.
    \item Derivative of $h(W_{(l)})$: the computation complexity is $O(P_l^2M^2_l)$. Therefore, for all the tasks, the computation complexity is $O(\sum P^2_lM^2_l)$.
    \item Derivative of $DCD(W_{(l_1)}, W_{(l_2)})$: We can preprocess $\frac{\partial l(W_{(l_1)})_{kl}}{\partial W_{(l_1)ij}}$ for all $i, j, k, l \in \{1,\ldots,P_{(l_1)}\}$. The complexity of this part is $O(P_{l_1}^6)$. Then for all $DCD(W_{(l_1)}, W_{(l_2)})$, we use $O(P_{l_1}P_{l_2}M_{l_1}M_{l_2})$ to compute its derivative. Therefore, for all pairs of tasks, the computation complexity is $O(\sum_{l_1}\sum_{l_2}P_{l_1}P_{l_2}M_{l_1}M_{l_2}+\sum_l P_l^6)$.
\end{itemize}

Denote $P=\max P_l$ and $M=\max M_l$; in each Adam iteration, the overall computation complexity is $O(LNM+L^2P^2M^2+LP^6)$. 

\begin{figure*}[h]
    \centering
    \includegraphics[width=0.99\linewidth]{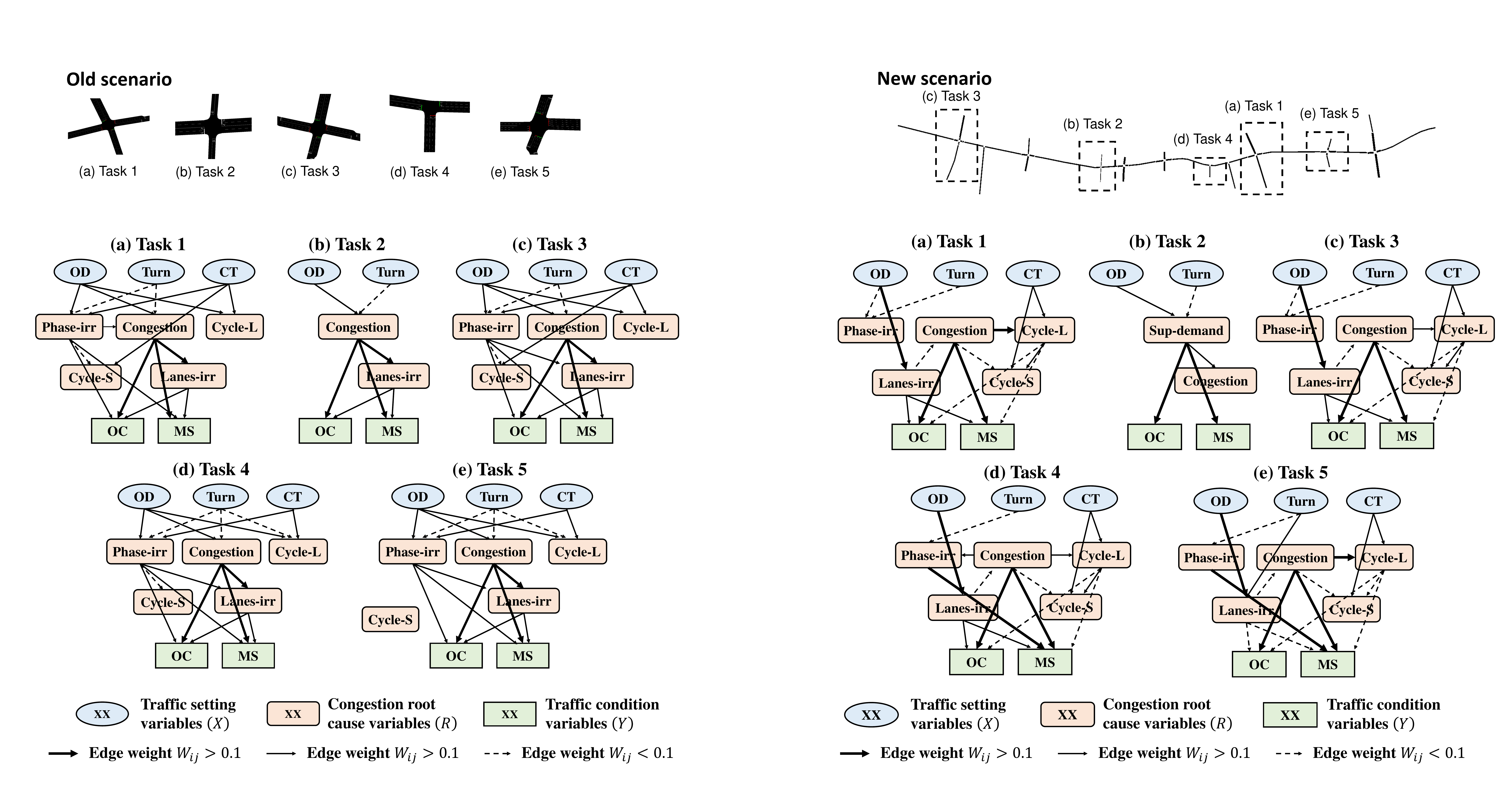}
    \caption{the MM-DAG results for the two cases}
    \label{fig:newsumoresult}
\end{figure*}

\section{Detailed result of numerical study} \label{app:comparison}

\begin{table}[th]
\small
\centering
\caption{\label{tab:compare} The detailed results of the numerical study with $N=20$ and $L=10$. The column labeled MM represents the multi-modal design, while the column labeled MT represents the multi-task method. The abbreviations OC, MD, and SE stand for Order-Consistency, Matrix-Difference, and Separate, respectively.}
\resizebox{\columnwidth}{!}{%
\begin{tabular}{cccccc}
\toprule
Model & MM? & MT &${\rm FDR}(\pm \rm std)$ &	$\rm TPR(\pm \rm std)$ & $\rm F1(\pm \rm std)$\\
\midrule
\textbf{MM-DAG} & Y & CD & $\mathbf{0.077\pm0.071}$ & $\mathbf{0.942 \pm0.075}$  & $\mathbf{0.929\pm0.051}$ \\
MV-DAG & N & CD & $0.545\pm0.047$ & $0.304 \pm0.107$  & $0.355\pm0.088$ \\
OC & Y & OC & $0.086\pm0.079$ & $0.821\pm0.085$ & $0.862\pm0.067$\\
MD & Y & MD & $0.302\pm0.140$ & $0.697\pm0.146$ & $0.691\pm0.126$\\
SE & Y & None & $0.344\pm0.123$   & $0.795\pm0.118$ & $0.709\pm0.097$\\
\bottomrule
\end{tabular}
}
\end{table}

Table \ref{tab:compare} presents a comprehensive overview of the numerical study with $N=20$ and $L=10$. In the following analysis, we will delve into the results and draw a conclusion based on the performance presented in the table.

\textbf{Explanation of the difference between MV-DAG and MM-DAG}: The MV-DAG approach cuts each functional data into a 10-dimensional vector by averaging the values within each of the 10 intervals. Compared to MM-DAG, MV-DAG has a 61.7\% lower F1 score, and we believe the reasons are twofold:

\begin{itemize}
    \item This preprocessing approach, working as a dimension reduction technique, may result in the loss of critical information of functional data.
    \item in MM-DAG, we delicately design a multimodal-to-multimodal (mulmo2) regression, which contains four carefully-designed functions, i.e., regular regression, func2vec regression, vec2func regression, and func2func regression (as shown in Fig. \ref{fig:MM-DAG}); whereas the MV-DAG only contains regular regression since all the function data have been vectorized.
\end{itemize}

\textbf{The contribution of CD design}: The performance of these three baselines (order-consistency, Matrix-Difference, Separate) in the new settings as in Table \ref{tab:compare} above. It is worth mentioning that all three baselines underwent the same multimodal-to-multimodal regression and got the same matrix $A$.

The table clearly indicates that our 'CD' design significantly contributed to improving the F1 score: MM-DAG has another +6.7\% F1 gain compared to order-consistency, as well as another +23.8\% F1-score gain compared to Matrix-Difference. These performance gains come purely from our CD design.

\textbf{The effectiveness of multitask learning}: By comparing MM-DAG with the baseline *Separate*, we show that it is essential to train the multiple overlapping but distinct DAGs in our multitask learning manner.

\textbf{Conclusion}: We compared our proposed MM-DAG model to the MV-DAG model to verify the contribution of the multi-modal design. Additionally, we compared MM-DAG to the \textit{Order-Consistence}, \textit{Matrix-Difference}, and \textit{Separate} models to demonstrate the effectiveness of our Causal Difference design. By combining these two comparisons, we have shown the effectiveness of both designs.

\section{New SUMO scenario} \label{app:newsumoscenario}

We constructed a more complex traffic scenario in SUMO, where 5 neighbor intersections in FengLinXi Road are used. In this case, the 5 intersections are not independent of the others. The detailed SUMO settings are as follows:
\begin{itemize}
    \item For the OD demand, we set OD demand as the number of total OD pairs in a scenario and randomly assign the origin and destination for each OD pair in the SUMO.
    \item For the turning probability, we calculate the turning vehicles at each intersection and divide by the total number of vehicles.
    \item The definition and collection of the remaining variables remain unchanged.
    \item In this new scenario, there is a new cause of congestion: [sup-demand], corresponding to OD demand exceeds the capacity of the intersection, as shown in Task 2 of our new results in Fig. \ref{fig:newsumoresult}. Yet this cause of congestion never occurred in the old scenario, so we did not plot this node in the DAGs of the old case study.
\end{itemize}

We have 96 samples in total, where each sample corresponds to a scenario in FengLinXi Road (Seeing Fig. \ref{fig:newsumoresult}). For each task, the data is collected by the sensors of the corresponding intersection. The result is shown in Fig. \ref{fig:newsumoresult}. 

We give an interpretation of the difference in DAGs between the old scenario and the new scenario. In the old scenario, the 5 intersections are \textbf{independent}. But in the new scenario, the 5 intersections are \textbf{cascaded and dependent}. The variable [OD-demand] is shared in all the DAGs since all the tasks used the same OD-demand. In the future revised manuscript, we will add both the \textbf{independent} case and \textbf{dependent} case, and the two cases have their own real-world applications.

\begin{itemize}
    \item \textbf{Independent} case: In the starting phase of deploying the traffic control systems, usually several single intersections are selected for the trial and cold-start. This trial period sometimes will last for more than one year and those intersections are usually scattered around different regions of a city.
    \item \textbf{Dependent} case: When the traffic signal control systems scale up and more intersections are signaled, sub-areas will be set up where up to eight intersections will be connected.
\end{itemize}

As we could observe in the Fig. \ref{fig:newsumoresult}:

\begin{itemize}
    \item The results of the two cases are different, which are reasonable given the two different assumptions.
    \item But we could still observe that the two results share quite consistent causal relations. For example, the thickest edges with weight > 0.5 are quite consistent in both independent and dependent cases.
    \item And we do admit that in the Dependent Case, the DAGs have unexpectedly better properties: (1) The DAGs are more sparse; (2) There are more shared edges across five different tasks. For example, the edge "Lane-Irrational" to "Congestion" appears in all five tasks.
\end{itemize}

\section{Potential Future Work}
\label{app: future}
In future work, we may like to try some deep learning methods. For example, we can consider incorporating layers able to deal with functional data \cite{yao2021deep}, and then extracting nonlinear features for all the nodes using graph neural network \cite{yu2019dag}.

\section*{Acknowledgements}

This paper was supported by the SenseTime-Tsinghua Research Collaboration Funding, NSFC Grant 72271138 and 71932006, the BNSF Grant 9222014, Foshan HKUST Projects FSUST20-FYTRI03B and the Tsinghua GuoQiang Research Center Grant 2020GQG1014.



\end{sloppypar}
\end{document}